\documentclass[10pt,twocolumn,letterpaper]{article}

\usepackage{cvpr}
\usepackage{times}
\usepackage{epsfig}
\usepackage{graphicx}
\usepackage{amsmath}
\usepackage{amssymb}

\usepackage{subfig}
\usepackage{graphicx}
\usepackage{array}
\usepackage{multirow}
\usepackage{diagbox}

\usepackage{amsthm}
\newtheorem{theorem}{Theorem}

\newtheorem{lemma}[theorem]{Lemma}
\theoremstyle{remark}

\usepackage{bm}
\usepackage{mathtools}
\DeclarePairedDelimiter\ceil{\lceil}{\rceil}

\numberwithin{equation}{section}
\usepackage{enumitem}
\usepackage{color}
\usepackage{authblk}

\cvprfinalcopy 

\def\cvprPaperID{4384} 

\ifcvprfinal\fi
\begin{document}
	
	\title{Fast Spatio-Temporal Residual Network for Video Super-Resolution}
	\newcommand*\samethanks[1][\value{footnote}]{\footnotemark[#1]}
	\author[1]{Sheng Li}
	\author[2]{Fengxiang He}
	\author[1]{Bo Du\thanks{Corresponding author.}}
	\author[1]{Lefei Zhang\samethanks}
	\author[3]{Yonghao Xu}
	\author[2]{Dacheng Tao}
	
	\affil[1]{School of Computer Science, Wuhan University, China}
	\affil[2]{UBTECH Sydney AI Centre, SCS, FEIT, the University of Sydney, Australia}
	\affil[3]{The State Key Laboratory of Information Engineering in Surveying, Mapping, and Remote Sensing, Wuhan University, China}
	\affil[ ]{\tt\small \{shli, remoteking, zhanglefei\}@whu.edu.cn  \{fengxiang.he, dacheng.tao\}@sydney.edu.au yonghaoxu@ieee.org}

	\maketitle
	\thispagestyle{empty}
	
	\begin{abstract}
    Recently, deep learning based video super-resolution (SR) methods have achieved promising performance. To simultaneously exploit the spatial and temporal information of videos, employing 3-dimensional (3D) convolutions is a natural approach. However, straight utilizing 3D convolutions may lead to an excessively high computational complexity which restricts the depth of video SR models and thus undermine the performance. In this paper, we present a novel fast spatio-temporal residual network (FSTRN) to adopt 3D convolutions for the video SR task in order to enhance the performance while maintaining a low computational load. Specifically, we propose a fast spatio-temporal residual block (FRB) that divide each 3D filter to the product of two 3D filters, which have considerably lower dimensions. Furthermore, we design a cross-space residual learning that directly links the low-resolution space and the high-resolution space, which can greatly relieve the computational burden on the feature fusion and up-scaling parts. Extensive evaluations and comparisons on benchmark datasets validate the strengths of the proposed approach and demonstrate that the proposed network significantly outperforms the current state-of-the-art methods.
	\end{abstract}
	
	\begin{figure}[tbp]
		\centering
		\subfloat[EDSR]{
			\begin{minipage}[t]{0.3\linewidth}
				\centering
				\includegraphics[height=2.5in]{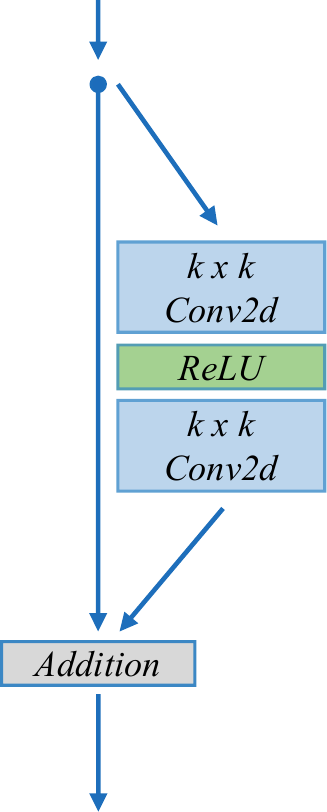} 
				\label{fig:Res_block_EDSR}
			\end{minipage}
		}
		\subfloat[3D residual block]{
			\begin{minipage}[t]{0.3\linewidth}
				\centering
				\includegraphics[height=2.5in]{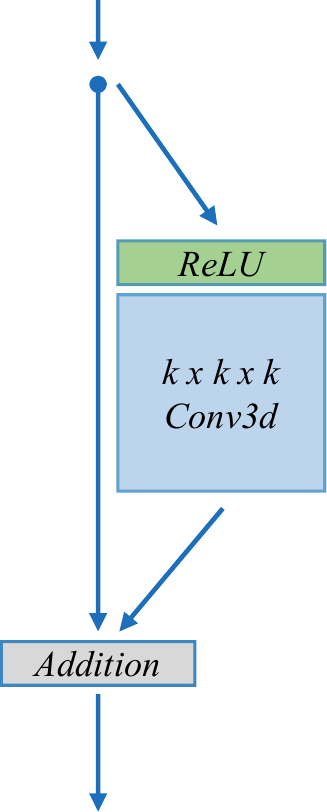}
				\label{fig:Res_block_Single_3D}
			\end{minipage}
		}
		\subfloat[Proposed FRB]{
			\begin{minipage}[t]{0.3\linewidth}
				\centering
				\includegraphics[height=2.5in]{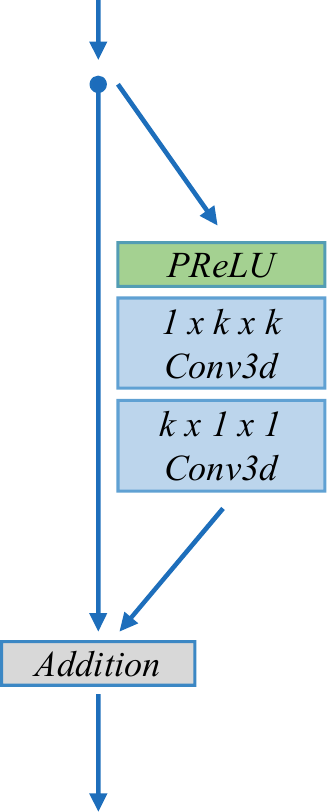}
				\label{fig:Res_block_Proposed}
			\end{minipage}
		}
		\centering
		\caption{Comparison of \protect\subref{fig:Res_block_EDSR} residual block in EDSR\cite{DBLP:conf/cvpr/LimSKNL17}, \protect\subref{fig:Res_block_Single_3D} single C3D residual block, and \protect\subref{fig:Res_block_Proposed} the proposed FRB.}
		\label{Res_block}
	\end{figure}
	
	\begin{figure*}[htbp]
		\centering
		\begin{minipage}[t]{\linewidth}
			\centering
			\includegraphics[width=1\linewidth]{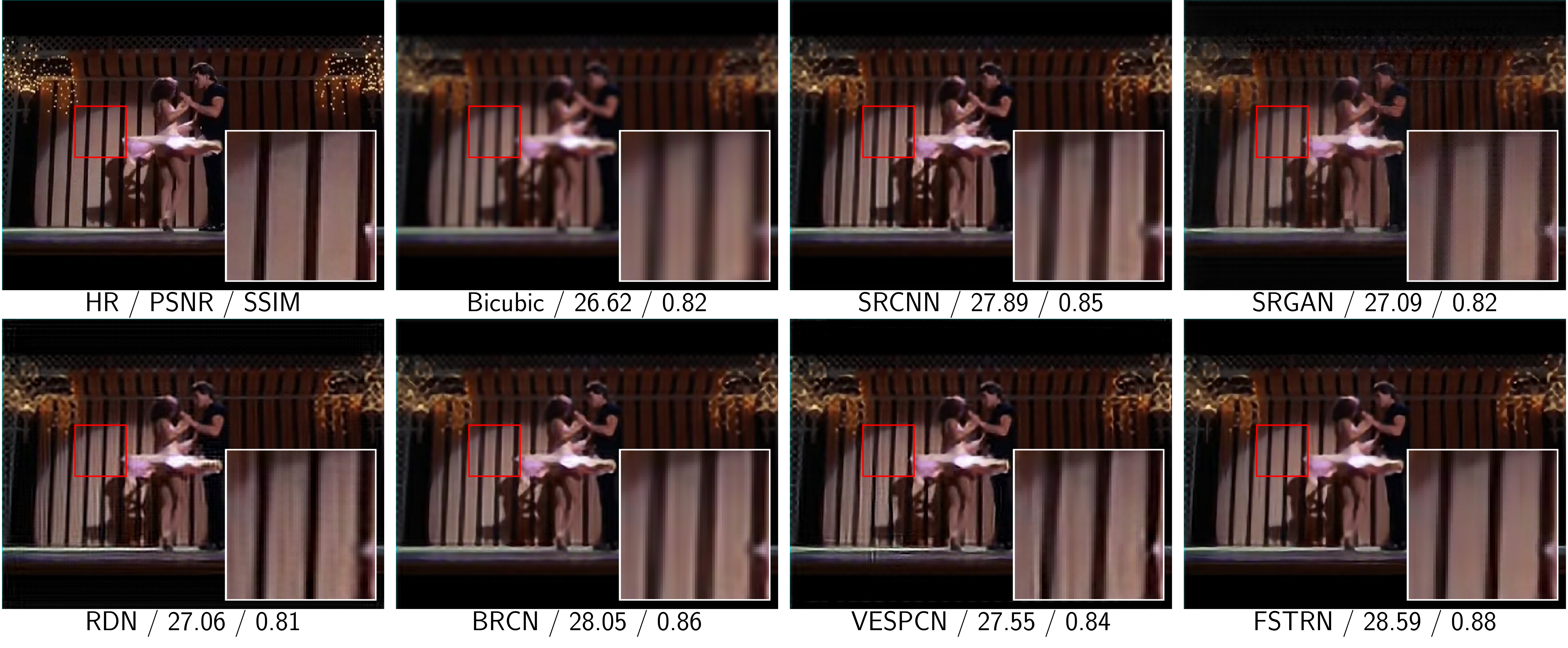} 
		\end{minipage}
		
		\caption{Visually observations on the orginal frames and the SR results on the Dancing video at $\times4$ SR, it is noticeable that the proposed FSTRN approach not only achieves the highest PSNR and SSIM values, but also restores the finest texture with the fewest artifacts.}
		\label{single_frame_comparison}
	\end{figure*}

	\section{Introduction}
	
	Super-resolution (SR) addresses the problem of estimating a high-resolution (HR) image or video from its low-resolution (LR) counterpart. SR is wildly used in various computer vision tasks, such as satellite imaging \cite{DBLP:conf/aaai/CaoJWL16} and surveillance imaging \cite{DBLP:conf/aaai/JiaoZWZG18}. 
	Recently, deep learning based methods have been a promising approach to solve SR problem \cite{DBLP:conf/eccv/DongLHT14,DBLP:conf/cvpr/KimLL16a,DBLP:conf/cvpr/LimSKNL17,DBLP:journals/tip/LiuWFLWCWH18,DBLP:journals/tip/LiuWWYHH16,DBLP:conf/iccv/WangLYHH15}. A straight idea for video SR is to perform single image SR frame by frame. However, it ignores the temporal correlations among frames, the output HR videos usually lack the temporal consistency, which may emerge as spurious flickering artifacts \cite{DBLP:conf/cvpr/ShiCHTABRW16}.
	
	Most existing methods for the video SR task utilize the temporal fusion techniques to extract the temporal information in the data, such as motion compensation \cite{DBLP:conf/cvpr/CaballeroLAATWS17,DBLP:conf/iccv/TaoGLWJ17}, which usually need manually deigned structure and much more computational consumption. To automatically and simultaneously exploit the spatial and temporal information, it is natural to employ 3-dimensional (3D) filters to replace 2-dimensional (2D) filters. However, the additional dimension would bring much more parameters and lead to an excessively heavy computational complexity. This phenomenon severely restricts the depths of the neural network adopted in the video SR methods and thus undermine the performance \cite{DBLP:journals/pami/HuangWW18}.
	
	
	Since there are considerable similarities between the input LR videos and the desired HR videos, the residual connection is widely involved in various SR networks \cite{DBLP:conf/cvpr/KimLL16a,DBLP:conf/cvpr/LedigTHCCAATTWS17,DBLP:conf/cvpr/LimSKNL17}, fully demonstrating the residual connection advantages. However, the residual identity mapping for SR task are beyond sufficient usage, it is either applied on HR space \cite{DBLP:conf/cvpr/KimLL16a,DBLP:conf/iccv/TaiYLX17}, largely increasing the computational complexity of the network, or applied on the LR space to fully retain the information from the original LR inputs \cite{DBLP:conf/icip/XuCSD18}, imposing heavy burdens on the feature fusion and upscaling stage at the final part of networks. 
	
	
	To address these problems, we propose fast spatio-temporal residual network (FSTRN) (Fig. \ref{fig:SR_architecture}) for video SR. It's difficult and impractical to build a very deep spatio-temporal network directly using original 3D convolution (C3D) due to high computational complexity and memory limitations. So we propose fast spatio-temporal residual block (FRB) (Fig. \ref{fig:Res_block_Proposed}) as the building module for FSTRN, which consists of skip connection and spatio-temporal factorized C3Ds. The FRB can greatly reduce computational complexity, giving the network the ability to learn spatio-temporal features simultaneously while guaranteeing computational efficiency.
	Also, global residual learning (GRL) are introduced to utilize the similarities between the input LR videos and the desired HR videos. On the one hand, we adopt to use LR space residual learning (LRL) in order to boost the feature extraction performance. On the other hand, we further propose a cross-space residual connection (CRL) to link the LR space and HR space directly. Through CRL, LR videos are employed as an ``anchor'' to retain the spatial information in the output HR videos.
	
	Theoretical analyses of the proposed method provide a generalization bound $\mathcal O(1/\sqrt{n})$ with no explicitly dependence on the network size ($n$ is the sample size), which guarantees the feasibility of our algorithm on unseen data. Thorough empirical studies on benchmark datasets evaluation validate the superiority of the proposed FSTRN over existing algorithms.
	
	In summary, the main contributions of this paper are threefold:
	\begin{itemize}
		\item We propose a novel framework fast spatio-temporal residual network (FSTRN) for high-quality video SR. The network can exploit spatial and temporal information simultaneously. By this way, we retain the temporal consistency and ease the problem of spurious flickering artifacts.  
		\item We propose a novel fast spatio-temporal residual block (FRB), which divides each 3D filter to the product of two 3D filters which have significantly lower dimensions. By this way, we significantly reduce the computing load while enhance the performance through deeper neural network architectures.
		\item We propose to employ global residual learning (GRL) which consist of LR space residual learning (LRL) and cross-space residual learning (CRL) to utilize the considerable similarity between the input LR videos and the output HR videos, which significantly improve the performance. 
	\end{itemize}
	
	\begin{figure*}[htbp]
		\centering
		\includegraphics[width=0.97\linewidth]{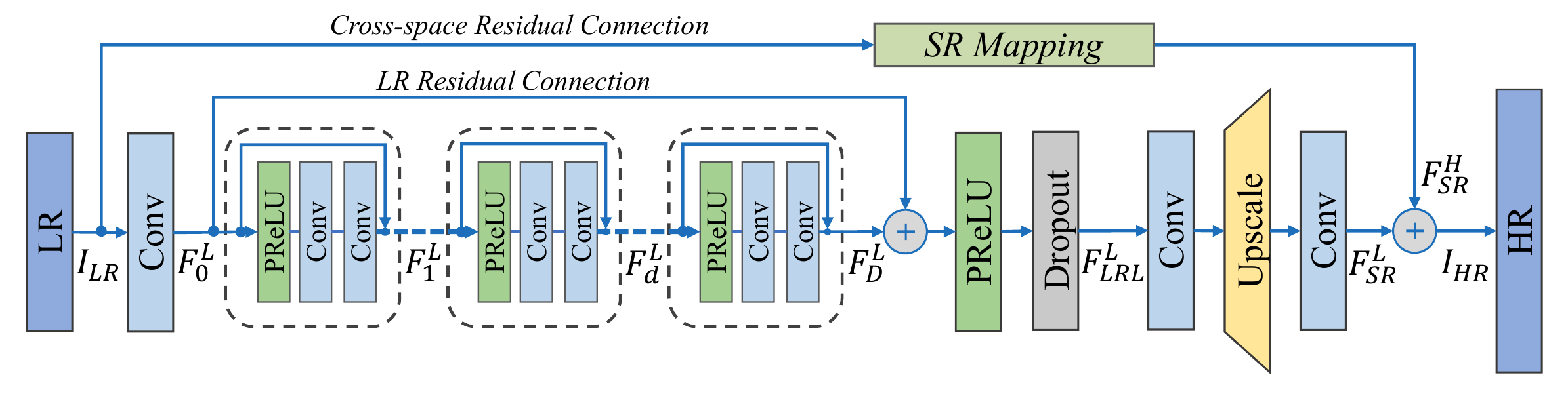}
		\caption{The architecture of our proposed fast spatio-temporal residual network (FSTRN).}
		\label{fig:SR_architecture}
	\end{figure*}
	
	\section{Related work}
	
	\subsection{Single-image SR with CNNs}
	
	In recent years, convolutional neural networks (CNNs) have achieved significant success in many computer vision tasks \cite{DBLP:conf/cvpr/HeZRS16,DBLP:conf/nips/KrizhevskySH12,lecun1998gradient,DBLP:journals/corr/SimonyanZ14a,DBLP:conf/cvpr/SzegedyLJSRAEVR15}, including the super-resolution (SR) problem. Dong \etal pioneered a three layer deep fully convolutional network known as the super-resolution convolutional neural network (SRCNN) to learn the nonlinear mapping between LR and HR images in the end-to-end manner \cite{DBLP:conf/eccv/DongLHT14,DBLP:journals/pami/DongLHT16}. Since then, many research has been presented, which are usually based on deeper network and more advanced techniques. 
	
	As the network deepens, residual connections have been a promising approach to relieve the optimization difficulty for deep neural networks \cite{DBLP:conf/cvpr/HeZRS16}. Combining residual learning, Kim \etal propose a very deep convolutional network \cite{DBLP:conf/cvpr/KimLL16a} and a deeply-recursive convolutional network (DRCN) \cite{DBLP:conf/cvpr/KimLL16}. These two models significantly boost the performance, which demonstrate the potentials of the residual learning in the SR task. Tai \etal present a deep recursive residual network (DRRN) with recursive blocks and a deep densely connected network with memory blocks \cite{DBLP:conf/iccv/TaiYLX17}, which further demonstrates the superior performance of residual learning.
	
	All the above methods work on interpolated upscaled input images. However, directly feeding interpolated images into neural networks can result in a significantly high computational complexity. To address this problem, an efficient sub-pixel convolutional layer \cite{DBLP:conf/cvpr/ShiCHTABRW16} and transposed convolutional layer \cite{DBLP:conf/eccv/DongLT16} are proposed in order to upscale the feature maps to a fine resolution at the end of the network.
	
	Other methods employing residual connections include EDSR \cite{DBLP:conf/cvpr/LimSKNL17}, SRResNet \cite{DBLP:conf/cvpr/LedigTHCCAATTWS17}, SRDenseNet \cite{DBLP:conf/iccv/0001LLG17} to RDN \cite{DBLP:conf/icip/XuCSD18}. However, residual connections are limited within the LR space. These residuals can enhance the performance of feature extraction but would put a excessively heavy load on the up-scaling and fusion parts of the network.

	\subsection{Video SR with CNNs}
	
	Based on image SR methods and further to grasp the temporal consistency, most existing methods employ a sliding frames window \cite{DBLP:conf/cvpr/CaballeroLAATWS17,jo2018deep,DBLP:journals/tci/KappelerYDK16,DBLP:conf/iccv/LiaoTLMJ15,DBLP:conf/iccv/TaoGLWJ17}. To handle spatio-temporal information simultaneously, existing methods usually utilize temporal fusion techniques, such as motion compensation \cite{DBLP:conf/cvpr/CaballeroLAATWS17,DBLP:journals/tci/KappelerYDK16,DBLP:conf/iccv/LiaoTLMJ15,DBLP:conf/iccv/TaoGLWJ17}, bidirectional recurrent convolutional networks (BRCN) \cite{DBLP:conf/nips/HuangWW15}, long short-term memory networks (LSTM) \cite{DBLP:conf/aaai/GuoC17}. Sajjadi \etal use a different way by using a frame-recurrent approach where the previous estimated SR frames are also redirected into the network, which encourages more temporally consistent results \cite{sajjadi2018frame}.
	
	A more natural approach to learn spatio-temporal information is to employ 3D convolutions (C3D), 
	which has shown superior performances in video learning \cite{DBLP:journals/pami/JiXYY13,DBLP:conf/iccv/TranBFTP15,DBLP:journals/corr/abs-1711-11248}. Caballero \etal \cite{DBLP:conf/cvpr/CaballeroLAATWS17} mentioned the slow fusion can also be seen as C3D. In addition, Huang \etal \cite{DBLP:journals/pami/HuangWW18} improved BRCN using C3D, allowing the model to flexibly obtain access to varying temporal contexts in a natural way, but the network is still shallow. In this work, we aimed to build a deep end-to-end video SR network with C3D and maintain high efficiency of computational complexity.

	\section{Fast spatio-temporal residual network}
	
	\subsection{Network structure}
	
	In this section, we describe the structure details of the proposed fast spatio-temporal residual network (FSTRN). As shown in Fig. \ref{fig:SR_architecture}, FSTRN mainly consists of four parts: LR video shallow feature extraction net (LFENet), fast spatio-temporal residual blocks (FRBs), LR feature fusion and up-sampling SR net (LSRNet), and global residual learning (GRL) part composing by LR space residual learning (LRL) and cross-space residual learning (CRL).
	
	\textbf{LFENet} simply uses a C3D layer to extract features from the LR videos. Let's denote the input and output of the FSTRN as $ I_{LR} $ and $ I_{SR}$ and the target output $ I_{HR} $, the LFENet can be represented as:
	\begin{equation}
	F_0^L = {H_{LFE}}\left({I_{LR}}\right),
	\label{eq:LFE}
	\end{equation}
	where $ F_0^L $ is the output of extracted feature-maps, and $ {H_{LFE}\left(\cdot\right)} $ denotes C3D operation in the LFENet. $ F_0^L $ is then used for later LR space global residual learning and also used as input to FRBs for further feature extraction. 
	
	\textbf{FRBs} are used to extract spatio-temporal features on the LFENet output. Assuming that $D$ of FRBs are used, the first FRB performs on the LFENet output, and the subsequent FRB further extract features on the previous FRB output, so the output $ F_d^L $ of the $d$-th FRB can be expressed as:
	\begin{equation}
	\begin{aligned}
	F_d^L & = H_{{FRB},d}\left(F_{d-1}^L\right)\\
	& = H_{{FRB},d}\left(H_{{FRB},d-1}\left(\cdots\left(H_{{FRB},1}\left(F_{0}^L\right)\right)\cdots\right)\right),
	\end{aligned}
	\label{eq:FRBd}
	\end{equation}
	where $ H_{{FRB},d} $ denotes the operations of the $d$-th FRB, more details about the FRB will be shown in Section \ref{sec:FRB}. 
	
	Along with the FRBs, LR space residual learning (LRL) is conducted to further improve feature learning in LR space. LRL makes fully use of feature from the preceding layers and can be obtained by
	\begin{equation}
	F_{LRL}^L = H_{LRL}\left(F_D^L, F_0^L\right),
	\label{eq:LRL_L}
	\end{equation}
	where $ F_{LRL}^L $ is the output feature-maps of LRL by utilizing a composite function $ H_{LRL} $. More details will be presented in Section \ref{sec:GRL}.
	
	\textbf{LSRNet} is applied to obtain super-resolved video in HR space after the efficient feature extraction of LRL. Specifically, we use a C3D for feature fusion followed by a deconvolution \cite{DBLP:journals/corr/DumoulinV16} for upscaling and again a C3D for feature-map channels tuning in the LSRNet. The output $ F_{SR}^L $ can be formulated as:
	\begin{equation}
	F_{SR}^L = H_{LSR}\left(F_{LRL}^L\right),
	\label{eq:LSR}
	\end{equation}
	where $ H_{LSR}\left(\cdot\right) $ denotes the operations of LSRNet.
	
	At last, the network output is composed of the $ F_{SR}^L $ from the LSRNet and an additional LR to HR space global residual, forming a cross-space residual learning (CRL) in HR space. The detail of the CRL is also given in Section \ref{sec:GRL}. So denote a SR mapping of input from LR space to HR space be $ F_{SR}^H $, the output of FSTRN can be obtained as
	\begin{equation}
	I_{SR} = H_{FSTRN}\left(I_{LR}\right) = F_{SR}^L + F_{SR}^H,
	\label{eq:FSTRN}
	\end{equation}
	where $ H_{FSTRN} $ represents the function of the proposed FSTRN method.
	
	\subsection{Fast spatio-temporal residual blocks}
	\label{sec:FRB}
	
	Now we present details about the proposed fast spatio-temporal residual block (FRB), which is shown in Fig. \ref{Res_block}.
	
	Residual blocks have been proven to show excellent performances in computer vision, especially in the low-level to high-level tasks \cite{DBLP:conf/cvpr/KimLL16a,DBLP:conf/cvpr/LedigTHCCAATTWS17}. 
	Lim \etal \cite{DBLP:conf/cvpr/LimSKNL17} proposed a modified residual block by removing the batch normalization layers from the residual block in SRResNet, as shown in Figure \ref{fig:Res_block_EDSR}, which showed a great improvement in single-image SR tasks. To apply residual blocks to multi-frame SR, we simply reserve only one convolutional layer, but inflate the 2D filter to 3D, which is similar to \cite{DBLP:journals/pami/JiXYY13}. As shown in Figure \ref{fig:Res_block_Single_3D}, the $k\times k$ square filter is expanded into a $k\times k\times k$ cubic filter, endowing the residual block with an additional temporal dimension.
	
	After the inflation, the ensuing problems are obvious, in that it takes much more parameters than 2D convolution, accompanied by more computations. To solve this, we propose a novel fast spatio-temporal residual block (FRB) by factorizing the C3D on the above single 3D residual block into two step spatio-temporal C3Ds, \ie, we replace the inflated $k\times k\times k$ cubic filter with a $1\times k\times k$ filter followed by a $k\times 1\times 1$ filter, which has been proven to perform better, in both training and test loss \cite{DBLP:journals/corr/abs-1711-11248,DBLP:journals/corr/abs-1712-04851}, as shown in Figure \ref{fig:Res_block_Proposed}. Also, we change the rectified linear unit (ReLU) \cite{DBLP:journals/jmlr/GlorotBB11} to its variant PReLU, in which the slopes of the negative part are learned from the data rather than predefined \cite{DBLP:conf/iccv/HeZRS15}. So the FRB can be formulated as:
	\begin{equation}
	F_{d}^L = F_{d-1}^{L} + W_{d,t}\left(W_{d,s}\left(\sigma\left(F_{d-1}^{L}\right)\right)\right),
	\label{eq:FRB}
	\end{equation}
	where $ \sigma $ denoted the PReLU \cite{DBLP:conf/iccv/HeZRS15} activation function. $ W_{d,s} $ and $ W_{d,t} $ correspond to weights of the spatial convolution and the temporal convolution in FRB, respectively, where the bias term is not shown.
	In this way, the computational cost can be greatly reduced, which will be shown in Section \ref{sec:FRB_exp}. Consequently, we can build a larger, C3D-based model to directly video SR under limited computing resources with better performance.
	
	\subsection{Global residual learning}
	\label{sec:GRL}
	
	In this section, we describe the proposed global residual learning (GRL) on both LR and HR space. For SR tasks, input and output are highly correlated, so the residual connection between the input and output is wildly employed. However, previous works either perform residual learning on amplified inputs, which would lead to high computational costs, or perform residual connection directly on the input-output LR space, followed by upscaling layers for feature fusion and upsamping, which puts a lot of pressure on these layers.
	
	To address these problems, we come up with global residual learning (GRL) on both LR and HR space, which mainly consists of two parts: LR space residual learning (LRL) and cross-space residual learning (CRL).
	
	\textbf{LR space residual learning} (LRL) is introduced along with the FRBs in LR space. We apply a residual connection with a followed parametric rectified linear unit (PReLU) \cite{DBLP:conf/iccv/HeZRS15} for it. Considering the high similarities between input frames, we also introduced a dropout \cite{DBLP:journals/jmlr/SrivastavaHKSS14} layer to enhance the generalization ability of the network. So the output $ F_{LRL}^L $ of LRL can be obtained by:
	\begin{equation}
	F_{LRL}^L = H_{LRL}\left(F_D^L, F_0^L\right) = \sigma_L\left(F_D^L+F_0^L\right),
	\label{eq:LRL}
	\end{equation}
	where $ \sigma_L $ denoted the combination function of PReLU activation and dropout layer.
	
	\textbf{Cross-space residual learning} (CRL) uses a simple SR mapping to directly map the LR video to HR space, and then adds to the LSRNet result $ F^L_{SR} $, forming a global residual learning in HR space. Specifically, CRL introduces a interpolated LR to the output, which can greatly alleviate the burden on the LSRNet, helping improve the SR results. The LR mapping to HR space can be represented as:
	\begin{equation}
	F_{SR}^H = H_{CRL}\left(I_{LR}\right),
	\label{eq:CRL}
	\end{equation}
	where $ F_{SR}^H $ is a super-resolved input mapping on HR space. $ H_{CRL} $ denotes the operations of the mapping function. The mapping function is selected to be as simple as possible so as not to introduce too much additional computational cost, including bilinear, nearest, bicubic, area, and deconvolution based interpolations.
	
	The effectiveness of GRL and the selection of SR mapping method is demonstrated in Section \ref{sec:Ablation_exp}. 
	
	\subsection{Network learning}
	\label{sec:Network_Learning}
	
	In training, we use $l_1$ loss function for training. To deal with the $l_1$ norm, we use the Charbonnier penalty function $ \rho \left( x \right)=\sqrt{{{x}^{2}}+{{\varepsilon }^{2}}} $ for the approximation. 
	
	Let $\theta $ be the parameters of network to be optimized, $ I_{SR} $ be the network outputs. Then the objective function is defined as:
	\begin{equation}
	\label{defObj}
	\mathcal{L}\left( I_{SR}, I_{HR};\theta \right)=\frac{1}{N}\sum\limits_{n=1}^{N}{\rho \left( {I}_{HR}^{ n }-{I}_{SR}^{ n } \right)}
	\end{equation} 
	where $N$ is the batch size of each training. Here we empirically set $\varepsilon =1e-3$. Note that although the network produces the same frames as the input, we focus on the reconstruction of the center frame from the input frames in this work. As a result, our loss function is mainly related to the center frame of the input frames.

	\section{Theoretical analysis}
	
	In learning theory, we usually use generalization error to express the generalization capability of an algorithm, which is defined as the difference between the expected risk $\mathcal R$ and the empirical risk $\hat{\mathcal R}$ of the algorithm. In this section, we study the generalization ability of FSTRN. Specifically, we first give an upper bound for the covering number $\mathcal N(\mathcal H)$ (covering bound) of the hypothesis space $\mathcal H$ induced by FSTRN. This covering bound constrain the complexity of FSTRN. Then we obtain an $O\left( \sqrt{\frac{1}{n}} \right)$ upper bound for the generalization error (generalization bound) of FSTRN. This generalization bound gives a theoretical guarantee to our proposed algorithms. 
	
	As Fig. \ref{fig:Res_block_Proposed} shows, FRB is obtained by adding an identity mapping to a chain-like neural network with one PReLU and two convolutional layers. Bartlett \etal proves that most standard nonlinearities are Lipschitz-continuous (including PReLU) \cite{bartlett2017spectrally}. Suppose the affine transformations introduced by the two convolutional operators can be respectively expressed by weight matrices $A_{1}^{i}$ and $A_{2}^{i}$. Expect all FRBs, from the input end of the stem to the output end, there are $1$ convolutional layer, $1$ PReLU, $1$ upscale, and $1$ convolutional layer (we don't consider dropout here). They can be respectively expressed by weight matrix $A_{1}$, nonlinearity $\sigma_{1}$, weight matrix $A_{2}$, and weight matrix $A_{3}$. As Fig. \ref{fig:SR_architecture} shows, LR residual learning is an identity mapping and HR residual learning can be expressed by a weight matrix $A_{HR}$. We can further obtain an upper bound for the hypothesis space induced by FSTRN as follows.

	\begin{theorem}[Covering bound for FSTRN]
		\label{thmCovBoundFSTRN}
		For the $i$-th FRB (i = 1, \ldots, D), suppose the Lipschitz constant of the PReLU is $\rho^{i}$, and the spectral norm of the weight matrices are bounded: $\|A_{1}^{i}\|_{\sigma} \le s_{1}^{i}$ and $\|A_{2}^{i}\|_{\sigma} \le s_{2}^{i}$. Also, suppose there are two reference matrices $M_{1}^{i}$ and $M_{2}^{i}$ respectively for $A_{1}^{i}$ and $A_{2}^{i}$, which are satisfied that $\| A_{i}^{i} - M_{i}^{i} \|_{\sigma} \le b_{i}^{i}$, $i = 1, 2$. Similarly, suppose the spectral norm of weight matrices $A_{1}$, $A_{2}$, $A_{3}$, and $A_{HR}$ are respectively upper bounded by $s_{1}$, $s_{2}$, $s_{3}$, and $s_{HR}$. Also, there are $4$ corresponding reference matrices $M_{i}$, $i \in \{ 1, 2, 3, HR \}$ such that $\| A_{i} - M_{i} \| \le b_{i}$. Meanwhile, suppose the Lipschitz constant of nonlinearity $\sigma_{1}$ is $\rho_{1}$. Then, the $\varepsilon$-covering number satisfies that 
		
		\begin{align}
		\label{eqCovBoundFRB}
		\mathcal N(\mathcal H) \le & \frac{b_{1}^{2} \| X \|^{2}_{2}\bar\alpha}{\varepsilon^{2}} \log\left(2W^{2}\right) + \sum_{d = 1}^{D} \mathcal N_{FRB}(d) \nonumber\\
		& + (*) \frac{b_{2}^{2}}{\varepsilon_{2}^{2}} \log\left(2W^{2}\right) \left[\left(\frac{b_{2}}{\varepsilon_{2}}\right)^{2} + \left(\frac{s_{2}b_{3}}{\varepsilon_{3}}\right)^{2} \right] \nonumber\\
		& + \frac{b_{HR}^{2} \| X \|^{2}_{2}}{\varepsilon^{2}}\log\left(2W^{2}\right),
		\end{align}
		
		where
		
		\begin{align}
		\mathcal N_{FRB}(d) = & \left(\frac{\| X \|_{2}s_{1}\rho^{d}}{\varepsilon^{d}}\right)^{2} \prod_{i = 1}^{d} \left[ \left( \rho^{i} s_{1}^{i} s_{2}^{i} \right)^{2} + 1 \right] \nonumber\\
		& \left[ \left(b_{1}^{d}\right)^{2} \left(1 + s_{2}^{d}\right)^{2} + \left(b_{2}^{d}s_{1}^{d} \right)^{2} \right],
		\end{align}
		\begin{equation}
		(*) = \left(\| X \|_{2} s_{1} \rho_{1} \right)^{2} \prod_{d = 1}^{D} \left[ \left( \rho^{d} s_{1}^{d} s_{2}^{d} \right)^{2} + 1 \right],
		\end{equation}
		\begin{equation}
		\varepsilon^{d} = \frac{\varepsilon - s_{HR} - 1}{\bar\alpha} \prod_{i=1}^{d} \left[ \rho^{i} (1 + s_{1}^{i}) (1 + s_{2}^{i}) + 1 \right],
		\end{equation}
		\begin{align}
		\varepsilon_2 = & \frac{\varepsilon - s_{HR} - 1}{\bar\alpha} \left\{ \prod_{i=1}^{D} \left[ \rho^{i} (1 + s_{1}^{i}) (1 + s_{2}^{i}) + 1 \right] + 1 \right\} \nonumber\\
		& \rho_{1} (1 + s_{2}) + s_{HR} + 1,
		\end{align}
		and
		\begin{align}
		\bar\alpha = \left\{ \prod_{j = 1}^{D} \left[ \rho^{j} (1 + s_{1}^{j}) (1 + s_{2}^{j}) + 1 \right] \right\} \rho_{1} (1 + s_{2}),
		\end{align}
		
	\end{theorem}
	
	A detailed proof is omitted here and given in the appendix based on \cite{bartlett2002rademacher, he2019why}. Finally, we can obtain the following theorem. For the brevity, we denote the right-hand side (RHS) of eq. (\ref{eqCovBoundFRB}) as $\frac{R}{\varepsilon}$.
	
	\begin{theorem}[Generalization Bound for FSTRN]
		\label{generalizartionBoundFSTRN}
		For any real $\delta \in (0, 1)$, with probability at least $1 - \delta$, the following inequality holds for any hypothesis $F_{\theta}$: 
		\begin{align}
		\label{generalizationBoundFSTRN}
		& \mathcal R (F_{\theta})\nonumber\\
		\le & \hat{\mathcal R} (F_{\theta}) + \frac{8}{N^{\frac{3}{2}}} + \frac{36}{N} \sqrt{R} \log N + 3 \sqrt{\frac{\log(2/\delta)}{2N}}.
		\end{align}
	\end{theorem}
	
	Theorem \ref{generalizartionBoundFSTRN} can be obtained from Theorem \ref{thmCovBoundFSTRN}. A detailed proof is given in the appendix. Eq. (\ref{generalizationBoundFSTRN}) gives an $O\left(1/\sqrt{N}\right)$ generalization bound for our proposed algorithm FSTRN. Another strength of our result is that all factors involved do not explicitly rely on the size of our neural network, which could be extremely large. This strength can prevent the proposed result from meaninglessness. Overall, this result theoretically guarantees the feasibility and generalization ability of our method.

	\section{Experiments}
	
	In this section, we first analyze the contributions of the network and then present the experimental results obtained to demonstrate the effectiveness of the proposed model on benchmark datasets quantitatively and qualitatively. 
	
	\subsection{Settings}
	
	\textbf{Datasets and metrics.} For a fair comparison with existing works, we used 25 YUV format benchmark video sequences as our training sets, which have been previously used in \cite{DBLP:conf/nips/HuangWW15,DBLP:journals/pami/HuangWW18,DBLP:journals/pami/LiuS14,DBLP:journals/tip/ProtterETM09,DBLP:journals/tip/TakedaMPE09}. We tested the proposed model on the benchmark challenging videos same as \cite{DBLP:conf/nips/HuangWW15} with the same settings, including the Dancing, Flag, Fan, Treadmill and Turbine videos, which contain complex motions with severe motion blur and aliasing. Following \cite{DBLP:conf/eccv/DongLHT14,DBLP:conf/iccv/TimofteDG13}, SR was only applied on the luminance channel (the Y channel in YCbCr color space), and performances were evaluated with the peak signal-to-noise ratio (PSNR) and structural similarity (SSIM) on the luminance channel.
	
	\begin{table*}[htbp]
		\centering
		\begin{tabular}{|c|c|c|c|c|c|c|}
			\hline
			\multirow{2}*{Methods}&Dancing&Treadmill&Flag&Fan&Turbine&Average\\
			& PSNR / SSIM & PSNR / SSIM & PSNR / SSIM & PSNR / SSIM & PSNR / SSIM & PSNR / SSIM\\
			\hline
			\hline
			Bicubic&26.78 / 0.83&21.58 / 0.65&26.97 / 0.78&33.42 / 0.93&26.06 / 0.76&27.80 / 0.80\\
			SRCNN\cite{DBLP:conf/eccv/DongLHT14}&27.91 / 0.87&22.61 / 0.73&28.71 / 0.83&34.25 / 0.94&27.84 / 0.81&29.20 / 0.84\\
			SRGAN\cite{DBLP:conf/cvpr/LedigTHCCAATTWS17}&27.11 / 0.84&22.40 / 0.72&28.19 / 0.83&33.48 / 0.93&27.38 / 0.81&28.65 / 0.84\\
			RDN\cite{DBLP:conf/icip/XuCSD18}&27.51 / 0.82&22.69 / 0.72&28.62 / 0.82&34.46 / 0.93&28.10 / 0.82&29.30 / 0.84\\
			BRCN\cite{DBLP:conf/nips/HuangWW15}&28.08 / 0.88&22.67 / 0.74&28.86 / 0.84&34.15 / 0.94&27.63 / 0.82&29.16 / 0.85\\
			VESPCN\cite{DBLP:conf/cvpr/CaballeroLAATWS17}&27.89 / 0.86&22.46 / 0.74&29.01 / 0.85&34.40 / 0.94&28.19 / 0.83&29.40 / 0.85\\
			\textbf{FSTRN}(ours)&\textbf{28.66} / \textbf{0.89}&\textbf{23.06} / \textbf{0.76}&\textbf{29.81} / \textbf{0.88}&\textbf{34.79} / \textbf{0.95}&\textbf{28.57} / \textbf{0.84}&\textbf{29.95} / \textbf{0.87}\\
			\hline
			
		\end{tabular}
		\caption{Comparison of the PSNR and SSIM results for the test video sequences by Bicubic, SRCNN\cite{DBLP:conf/eccv/DongLHT14}, SRGAN\cite{DBLP:conf/cvpr/LedigTHCCAATTWS17}, RDN\cite{DBLP:conf/icip/XuCSD18}, BRCN\cite{DBLP:conf/nips/HuangWW15}, VESPCN\cite{DBLP:conf/cvpr/CaballeroLAATWS17}, and our FSTRN with scale factor $4$.}
		\label{tab:psnr_ssim}
	\end{table*}
	
	\begin{figure}[tbp]
		\subfloat[]{
				\includegraphics[width=0.48\linewidth]{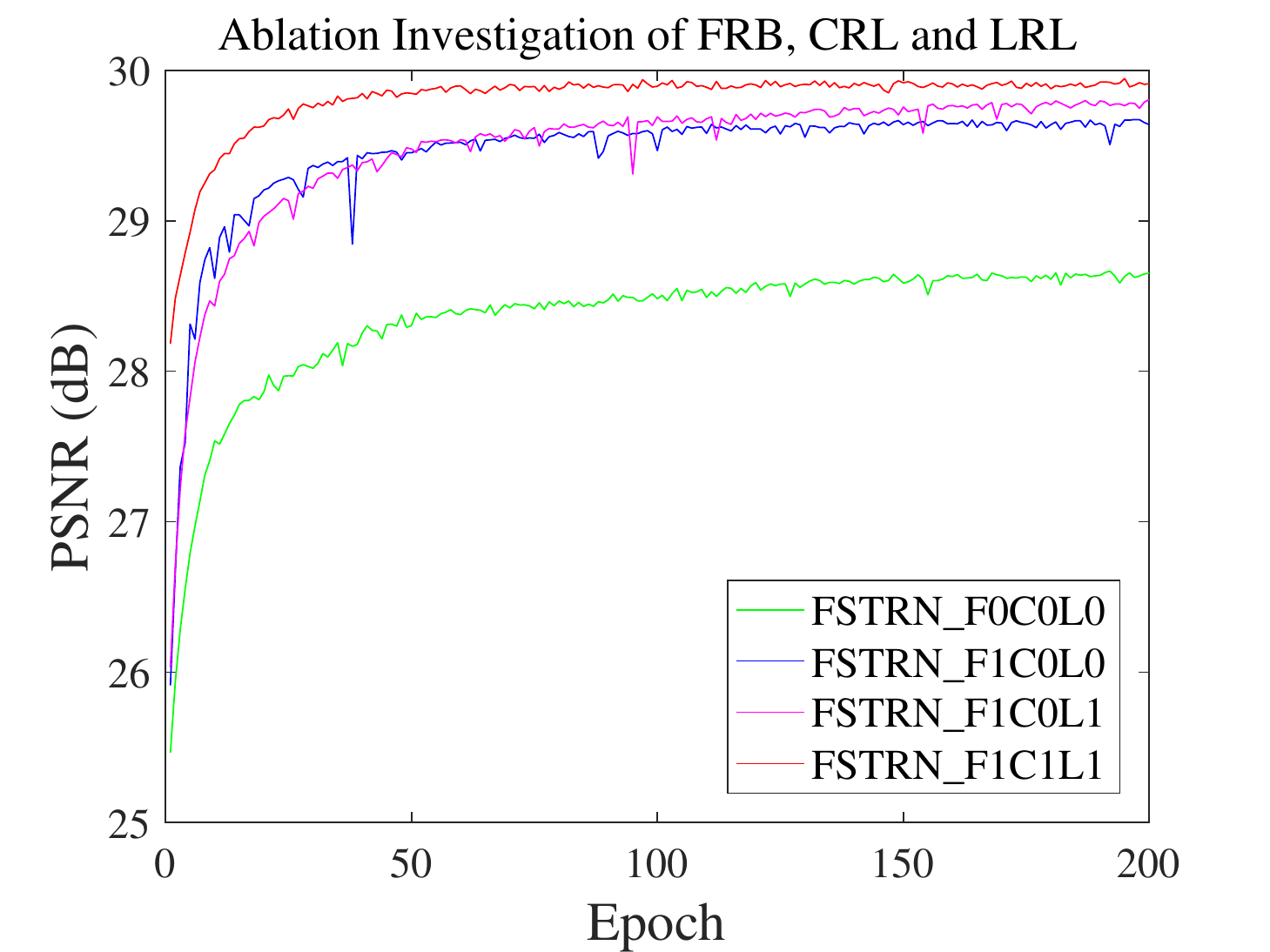} 
				\label{fig:ablation_module}
		}
		\subfloat[]{
				\includegraphics[width=0.48\linewidth]{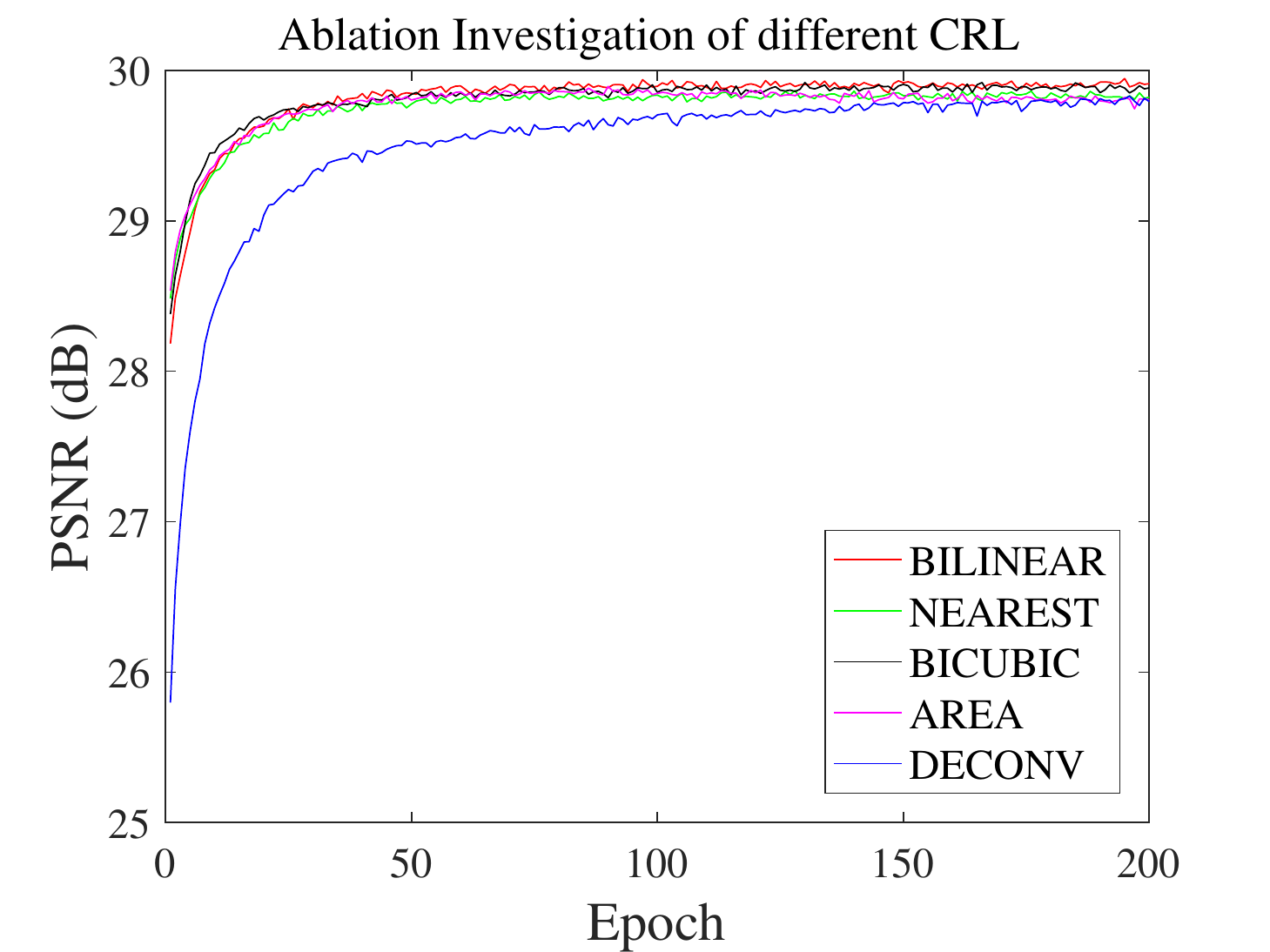}
				\label{fig:upscale_compare}
		}
		\centering 
		\caption{Convergence analysis on different degradation models \protect\subref{fig:ablation_module} and different interpolation method for CRL \protect\subref{fig:upscale_compare}. The curves for each combination are based on the PSNR on test video with scaling factor $\times 4$ in 200 epochs.}
		\label{fig:Convergence}
	\end{figure}
	
	\textbf{Training settings.} Data augmentation was performed on the 25 YUV video sequences dataset. Following \cite{DBLP:conf/nips/HuangWW15,DBLP:journals/pami/HuangWW18}, to enlarge the training set, we trained the model in a volume-based way by cropping multiple overlapping volumes from the training videos. During the cropping, we took a large spatial size as $144\times144$ and the temporal step as $5$, and the spatial and temporal strides were set as 32 and 10, respectively. Furthermore, inspired by \cite{DBLP:conf/cvpr/TimofteRG16}, the flipped and transposed versions of the training volumes were considered. Specifically, we rotated the original images by $90^{\circ}$ and flipped them horizontally and vertically. As a result, we could generate 13020 volumes from the original video dataset. After this, both of the training and testing LR inputs generating processes are divided into two stages: smoothing each original frame by a Gaussian filter with a standard deviation of 2, and downsampling the preceding frames using the bicubic method. In addition, to maintain the number of output frames equal to original video in the test stage, frame padding was applied at the test videos head and tail.
	
	In these experiments, we focused on video SR of upscale factor 4, which is usually considered the most challenging and universal case in video SR. The number of FRBs and the dropout rate were empirically set to be $5$ and $0.3$. The Adam optimizer \cite{DBLP:journals/corr/KingmaB14} was used to minimize the loss function with standard back-propagation. We started with a step size of $1e-4$ and then reduced it by a factor of 10 when the training loss stopped going down. The batch size was set depending on the GPU memory size.

	\subsection{Study of FRB}
	\label{sec:FRB_exp}
	
	In this section, we investigate the effect of the proposed FRB on efficiency. We analyze the computational efficiency of the FRB compared to the residual block built directly using C3D (C3DRB). Supposing we have all input and output feature-map size of $ 64 $, each input consists $ 5 $ frames with the size $ 32\times32 $, then a detail params and floating-point operations (FLOPs) comparison of the proposed FRB and the C3DRB are summarized in Table \ref{tab:FRB_exp}. It's obvious to see that the FRB can greatly reduce parameters and calculations by more than half amount. In this way, the computational cost can be greatly reduced, so we can build a larger, C3D-based model to directly video SR under limited computing resources with better performance.
	
	\begin{table}[h]
		\renewcommand{\arraystretch}{1.3}
		\centering
		\begin{tabular}{c|c|c}
			\hline
			Blocks&\#Params&\#FLOPs\\
			\hline
			\hline
			C3DRB& $\sim 111$K & $\sim 566$M \\
			
			FRB& $\sim 49$K & $\sim 252$M \\
			\hline
			Reduce ratio& 55.86\% & 55.48\% \\
			\hline
		\end{tabular}
		\caption{\#Params and \#FLOPs comparisons of one residual block using single C3D (Fig. \ref{fig:Res_block_Single_3D}) and one FRB (Fig. \ref{fig:Res_block_Proposed}).}
		\label{tab:FRB_exp}
	\end{table}
	
	\subsection{Ablation investigations}
	\label{sec:Ablation_exp}
	
	We conducted ablation investigation to analyze the contributions of FRBs and GRL with different degradation models in this section. Fig. \ref{fig:ablation_module} shows the convergence curves of the degradation models, including: 1) the baseline obtained without FRB, CRL and LRL (FSTRN\_F0C0L0); 2) baseline integrated with FRBs (FSTRN\_F1C0L0); 3) baseline with FRBs and LRL (FSTRN\_F1C0L1); 4) baseline with all components of FRBs, CRL and LRL (FSTRN\_F1C1L1), which is our FSTRN. The number $D$ of FRBs was set to $5$, and CRL uses bilinear interpolation.
	
	The baseline converges slowly and performs relatively poor (green curve), and the additional FRBs greatly improve the performance (blue curve), which can be due to the efficient inter-frame features capture capabilities. As expected, LRL further improved network performance (magenta curve). Finally, the addition of CRL was applied (red curve), constituted GRL on both LR and HR space. It can be clearly seen that the network performed faster convergence speed and better performance, which demonstrated the effectiveness and superior ability of FRB and GRL.
	
	Furthermore, to show how different interpolation methods in CRL affect the network performance, we investigated different interpolation method for CRL. Specifically, we explored bilinear, nearest, bicubic, area and deconvolution based interpolations. As shown in Fig. \ref{fig:upscale_compare}, different interpolation method except deconvolution behaved almost the same, reason for this is because the deconvolution needs a process to learn the upsampling filers, while other methods do not need. All the different interpolation method converged to almost the same performance, indicated that the performance improvement of FSTRN is attributed to the introduction of GRL, and has little to do with specific interpolation method in CRL.

	\begin{figure*}[htbp]
		\centering
		\subfloat[Original]{
			\begin{minipage}[t]{0.143\linewidth}
				\centering
				\includegraphics[height=4in]{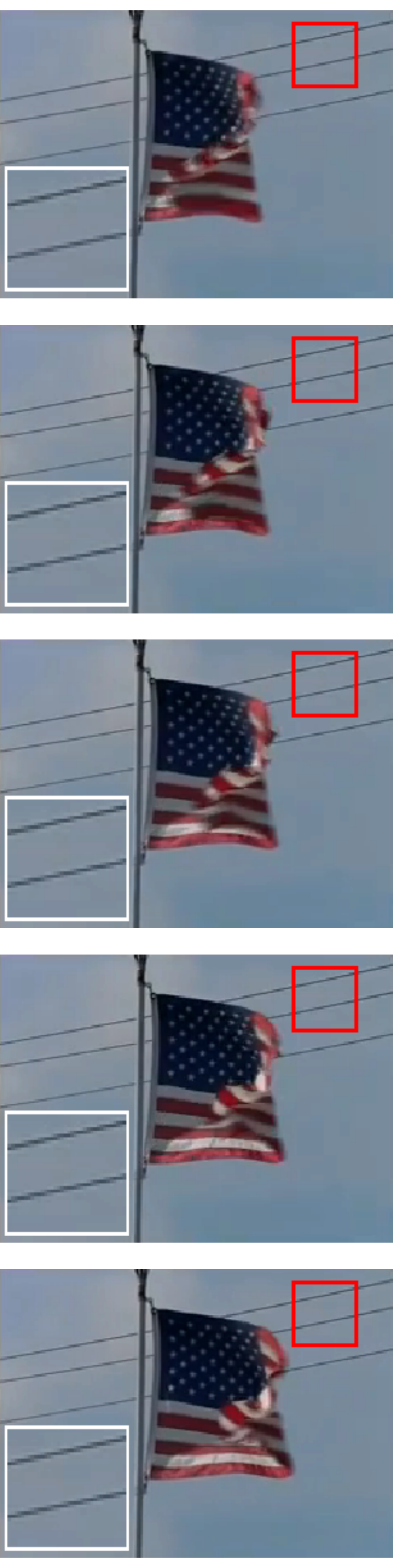} 
				\label{fig:Original}
			\end{minipage}
		}
		\subfloat[SRCNN]{
			\begin{minipage}[t]{0.143\linewidth}
				\centering
				\includegraphics[height=4in]{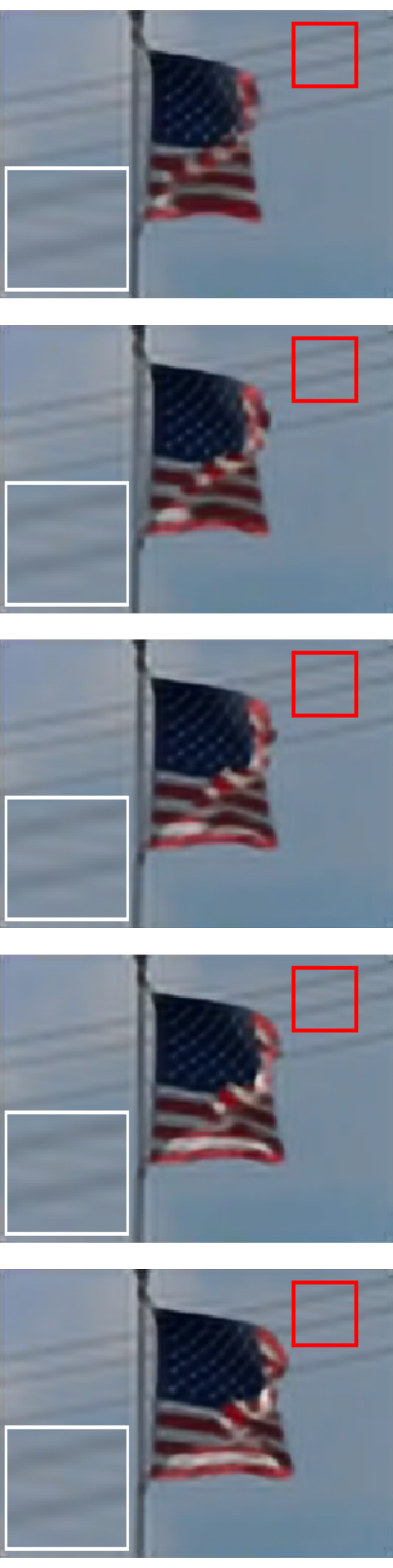}
				\label{fig:SRCNN}
			\end{minipage}
		}
		\subfloat[RDN]{
			\begin{minipage}[t]{0.143\linewidth}
				\centering
				\includegraphics[height=4in]{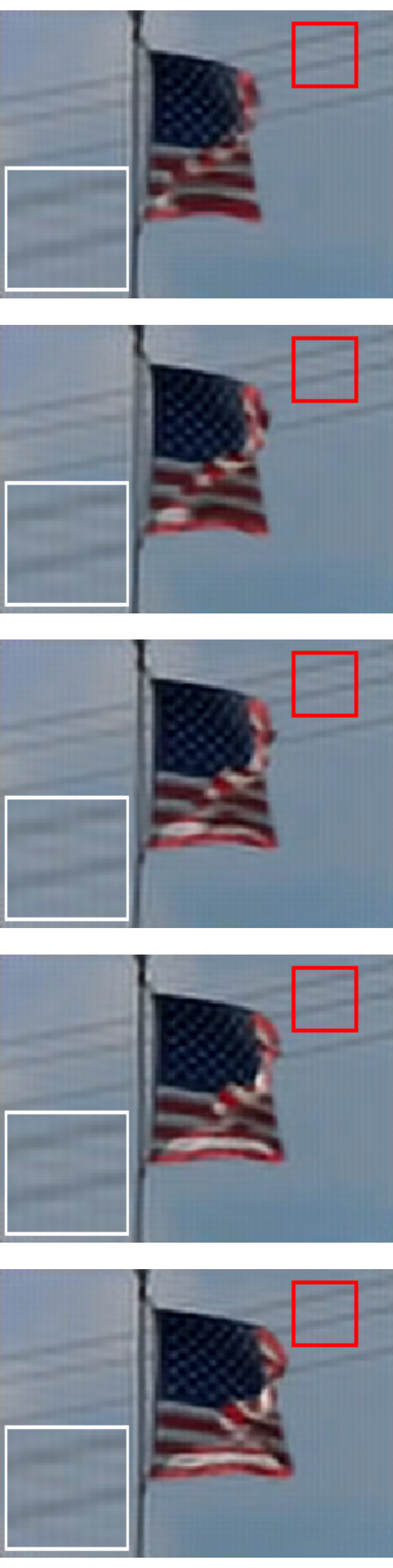}
				\label{fig:RDN}
			\end{minipage}
		}
		\subfloat[BRCN]{
			\begin{minipage}[t]{0.143\linewidth}
				\centering
				\includegraphics[height=4in]{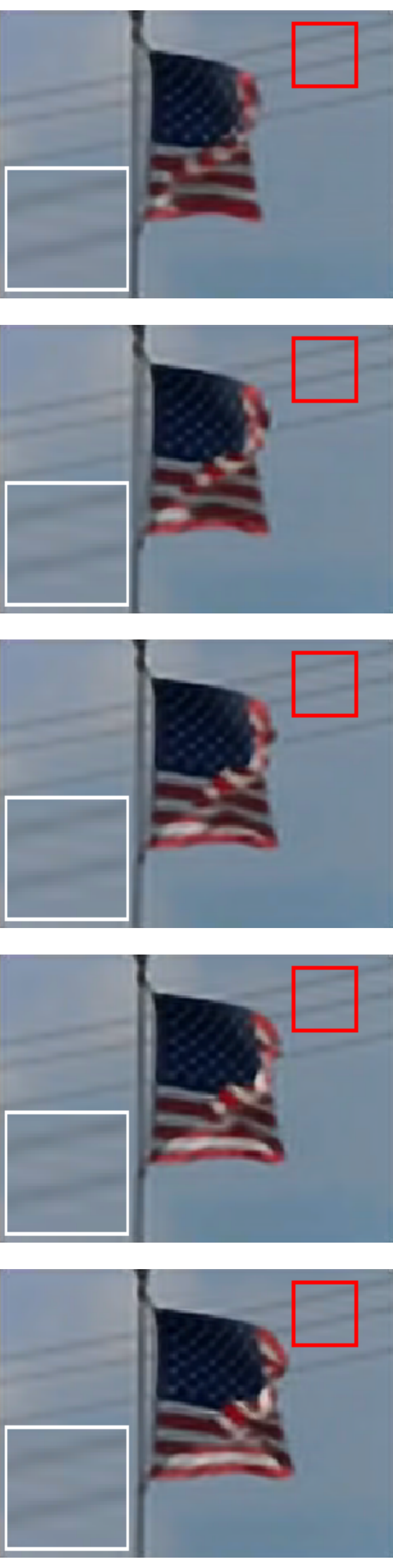}
				\label{fig:BRCN}
			\end{minipage}
		}
		\subfloat[VESPCN]{
			\begin{minipage}[t]{0.143\linewidth}
				\centering
				\includegraphics[height=4in]{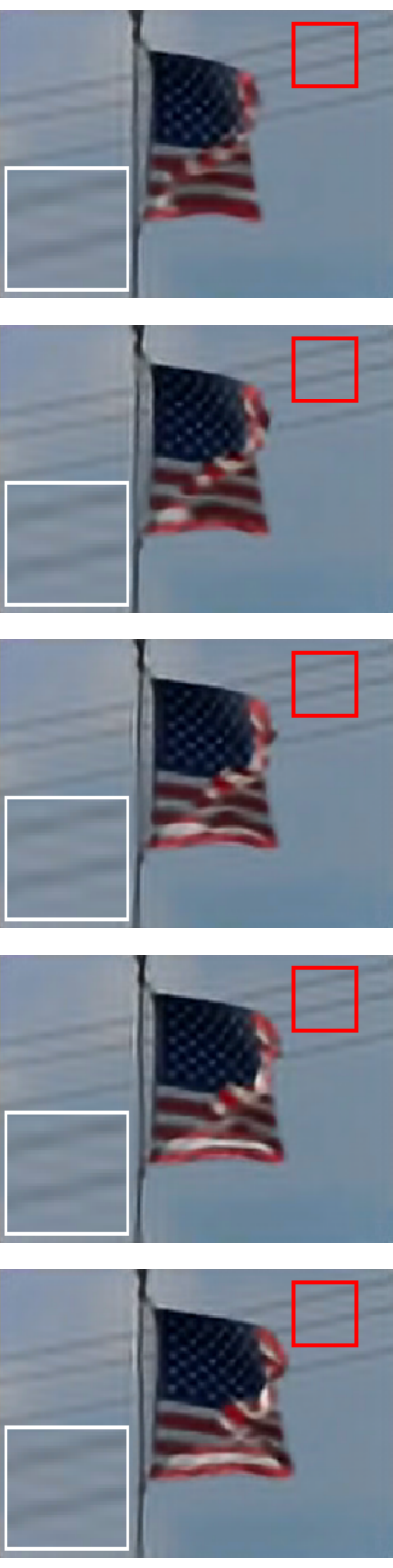}
				\label{fig:VESPCN}
			\end{minipage}
		}
		\subfloat[FSTRN]{
			\begin{minipage}[t]{0.143\linewidth}
				\centering
				\includegraphics[height=4in]{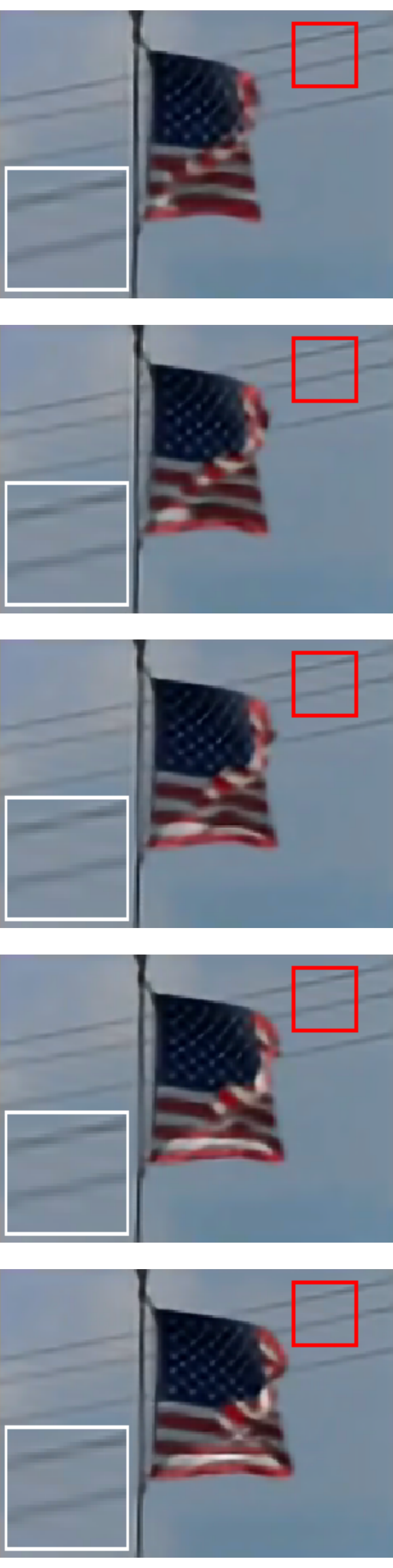}
				\label{fig:FSTRN}
			\end{minipage}
		}
		\centering
		\caption{Comparison between original frames ($1{st}\sim5{th}$ frames, from the top row to bottom) of the Flag video and the SR results obtained by SRCNN, RDN, BRCN, VESPCN and FSTRN, respectively. Our results show sharper outputs with smoother inter-frame transitions compared to other works.}
		\label{multi_frame_comparison}
	\end{figure*}
	
	\subsection{Comparisons with state-of-the-art}
	
	We compared the proposed method with different single-image SR methods and state-of-the-art multi-frame SR methods, both quantitatively and qualitatively, including Bicubic interpolation, SRCNN \cite{DBLP:conf/eccv/DongLHT14,DBLP:journals/pami/DongLHT16}, SRGAN \cite{DBLP:conf/cvpr/LedigTHCCAATTWS17}, RDN \cite{DBLP:conf/icip/XuCSD18}, BRCN \cite{DBLP:conf/nips/HuangWW15,DBLP:journals/pami/HuangWW18} and VESPCN \cite{DBLP:conf/cvpr/CaballeroLAATWS17}. The number $D$ of FRBs was set to $5$ in following comparisons and the upscale method of CRL was set to bilinear interpolation.
	
	The quantitative results of all the methods are summarized in Table \ref{tab:psnr_ssim}, where the evaluation measures are the PSNR and SSIM indices. Specifically, compared with the state-of-the-art SR methods, the proposed FSTRN shows significant improvement, surpassing them 0.55 dB and 0.2 on average PSNR and SSIM respectively.
	
	In addition to the quantitative evaluation, we present some qualitative results in terms of single-frame (in Figure \ref{single_frame_comparison}) and multi-frame (in Figure \ref{multi_frame_comparison}) SR comparisons, showing visual comparisons between the original frames and the $\times4$ SR results. It is easy to see that the proposed FSTRN recovers the finest details and produces most pleasing results, both visually and with regard to the PSNR/SSIM indices. Our results show sharper outputs and even in grid processing, which is recognized as the most difficult to deal in SR, the FSTRN can handle it very well, showing promising performance.
	
	\section{Conclusion}
	
	In this paper, we present a novel fast spatio-temporal residual network (FSTRN) for video SR problem. We also design a new fast spatio-temporal residual block (FRB) to extract spatio-temporal features simultaneously while assuring high computational efficiency.
	Besides the residuals used on the LR space to enhance the feature extraction performance, we further propose a cross-space residual learning to exploit the similarities between the low-resolution (LR) input and the high-resolution (HR) output. Theoretical analysis provides guarantee on the generalization ability, and empirical results validate the strengths of the proposed approach and demonstrate that the proposed network significantly outperforms the current state-of-the-art SR methods. 
	
	\section{Acknowledgements}
	
	This work was supported in part by the National Natural Science Foundation of China under Grants 61822113, 41871243, 41431175, 61771349, the National Key R \& D Program of China under Grant 2018YFA0605501, Australian Research Council Projects FL-170100117, DP-180103424, IH-180100002 and the Natural Science Foundation of Hubei Province under 2018CFA050.
	
	{\small
		\bibliographystyle{ieee}
		\bibliography{fstrnbib}
	}





\newpage

\begin{appendix}
	
	
	
	
	\section{Proof}
	\label{sec:sup_proof}
	
	This appendix collects all the proofs omitted from the main text.
	
	\subsection{Preliminary}
	\label{pre}
	
	This subsection gives the background knowledges necessary to the development of the theoretical analysis.
	
	A tuned FSTRN induces a hypothesis function that maps from low-resolution videos to high-resolution videos. For the brevity, we denote the hypothesis function as
	\begin{align}
	F_{\theta}: ~ & \mathbb R^{n_{LR}} \to \mathbb R^{n_{HR}}, \\
	& I_{LR} \mapsto I_{HR},
	\end{align}
	where $\theta$ is the tuned parameter, and $n_{LR}$ and $n_{LR}$ are respectively the dimensions of the low-resolution space and the high-resolution space. Suppose all the hypothesis functions $F_{\theta}$ computed by FSTRN constitute a hypothesis space $\mathcal H$. To measure the performance of the hypothesis function, we define an object function in the main text as eq. (\ref{defObj}). The corresponding loss function is defined as follows: 
	\begin{align}
	l\left( I_{SR}, I_{HR};\theta \right) = & \rho \left( {I}_{HR}-{I}_{SR} \right) \nonumber\\
	& = \sqrt{{\left({{I}_{HR}-{I}_{SR}}\right)^{2}}+{{\varepsilon }^{2}}},
	\end{align}
	where ${I}_{HR}$ and ${I}_{LR}$ are respectively the output (high-resolution image/video) and input (low-resolution image/video), and $ \rho \left( x \right)=\sqrt{{{x}^{2}}+{{\varepsilon }^{2}}} $ is Charbonnier penalty function. Based on the loss function $l(I_{SR}, I_{HR}, F_{\theta})$, the expected risk, in term of the hypothesis function $F_{\theta}$, is defined as follows:
	\begin{equation}
	\mathcal R (F_{\theta}) = \mathbb E_{I_{SR}, I_{HR}} l(I_{SR}, I_{HR}, F_{\theta}).
	\end{equation}
	Similarly, the empirical risk is defined as
	\begin{equation}
	\hat{\mathcal R} (F_{\theta}) = \mathcal L (F_{\theta}) = \frac{1}{N} \sum_{n=1}^{N} l(I_{SR}^{n}, I_{HR}^{n}, F_{\theta}),
	\end{equation}
	where $I_{SR}^{n}$ and $I_{HR}^{n}$ denote the $n$-th instance in the training set, and $N$ is the sample size, and we redefine the empirical risk as $\hat{\mathcal R}$ in accordance with the convention. Finally, the generalization error of hypothesis function $F(\theta)$ is defined as the difference between the expected risk $\mathcal R (F_{\theta})$ and the corresponding empirical risk 
	$\hat{\mathcal R} (F_{\theta})$.
	
	As the principle of {\it Occam's razor} says, the generalization capability of an algorithm is dependent with the complexity of its corresponding hypothesis space (hypothesis complexity): a complex algorithms tend to have a poor generalization ability. In learning theory, three classic measurements of hypothesis complexity are respectively VC-dimension, Rademacher complexity, and covering number (see, respectively, \cite{bartlett2002rademacher},  \cite{vapnik1974theory}, and \cite{dudley2010sizes}). An classic result in learning theory expresses the negative correlation between the generalization error of an algorithm and the corresponding Rademacher complexity $\hat{\mathfrak R}(\mathcal H)$ as the following lemma.
	\begin{lemma}[cf. \cite{mohri2012foundations}, Theorem 3.1]
		\label{GenBound}
		For any $\delta > 0$, with probability at least $1 - \delta$, the following inequality hold for all $F_{\theta} \in \mathcal H$:
		\begin{equation}
		\mathcal R (F_{\theta}) \le \hat{\mathcal R} (F_{\theta}) + 2 \hat{\mathfrak R}(l \circ \mathcal H) + 3 \sqrt{\frac{\log \frac{2}{\delta}}{2N}},
		\end{equation}
		where $l \circ \mathcal H$ is defined as
		\begin{equation}
		l \circ \mathcal H \triangleq \{ l \circ F: ~ F \in \mathcal H \}.
		\end{equation}
	\end{lemma}
	
	Computing the empirical Rademacher complexity of neural network could be extremely difficult and thus still remains an open problem. Fortunately, the empirical Rademacher complexity can be upper bounded by the corresponding $\varepsilon$-covering number $N(\mathcal H, \varepsilon, \| \cdot \|_{2})$ as the following lemma states.
	\begin{lemma}[cf. \cite{bartlett2017spectrally}, Lemma A.5]
		\label{covNumBound}
		Suppose $\bm 0 \in \mathcal H$ and all conditions in Lemma \ref{GenBound} hold. Then
		\begin{align}
		\label{formulaCovNumBound}
		& \hat{\mathfrak R}(\mathcal H) \nonumber\\
		& \le \inf_{\alpha > 0} \left( \frac{4\alpha}{\sqrt{n}} + \frac{12}{n} \int_{\alpha}^{\sqrt n} \sqrt{\log \mathcal N(l \circ \mathcal H, \varepsilon, \| \cdot \|_{2})} \text{d}\varepsilon \right).
		\end{align}
	\end{lemma}
	
	Some recent works study the hypothesis complexity of deep neural networks and provide upper bounds of the corresponding hypothesis spaces. \cite{bartlett2017spectrally} gives a spectrally-normalised covering bound and a generalization bound for all chain-like neural networks. \cite{he2019why} focuses on the deep neural networks with shortcut connections and gives a covering bound and a corresponding generalization bound. Specifically, for a deep neural network with residual connections, suppose the ``stem'' is obtained by discarding all residual connections. Apparently, it is a chain-like neural network and can be expressed by the following formula:
	\begin{equation}
	\label{stemStructure}
	S = (A_{1}, \sigma_{1}, A_{2}, \sigma_{2}, \ldots, A_{L}, \sigma_{L}),
	\end{equation}
	where $A_{i}$, $i = 1, \ldots, L$ are weight matrices and $\sigma_{i}$ are nonlinearities. Meanwhile, we denote all residual connections as $V_{j}$, $j \in J$, where $J$ is the index set. Suppose the output of the $i$-th layer (constituted by the weight matrix $A_{i}$ and the nonlinearity $\sigma_{i}$) is $F_{i}$, and all possible outputs $F_{i}$ constitute a hypothesis space $\mathcal H_{i}^{S}$. Similarly, all outputs of the residual connection $V_{j}$ constitute a hypothesis space $\mathcal H_{j}^{V}$. In this paper, our theoretical analysis is developed based on the two works stated above. Specifically, the covering bounds given by \cite{he2019why, bartlett2017spectrally} are respectively as follows.
	
	\begin{lemma}[see \cite{he2019why}, Theorem 1]
		\label{coverBoundGeneral}
		Suppose the $\varepsilon_{i}^{S}$-covering number of $\mathcal H_{i}^{S}$ is $\mathcal N_{i}^{S}$ and the $\varepsilon_{j}^{V}$-covering number of $\mathcal H_{j}^{V}$ is $\mathcal N_{i}^{V}$. Then there exists an $\varepsilon$ in terms of all $\varepsilon_{i}^{S}$ and $\varepsilon_{j}^{V}$, such that the following inequality holds:
		\begin{align}
		\label{formulaCoverBoundGeneral}
		\mathcal N(\mathcal H) \le \prod_{i = 1}^{L} \mathcal N_{i}^{S} \prod_{j \in J} \mathcal N^{j}_{V}.
		\end{align}
	\end{lemma}
	
	\begin{lemma}[cf. \cite{bartlett2017spectrally}, Lemma A.7]
		\label{coverChainBoundGeneral}
		Suppose there are $L$ weight matrices in a chain-like neural network. Let $(\varepsilon_{1}, \ldots, \varepsilon_{L})$ be given. Suppose the $L$ weight matrices $(A_{1}, \ldots, A_{L})$ lies in $\mathcal B_{1} \times \ldots \times \mathcal B_{L}$, where $\mathcal B_{i}$ is a ball centered at $0$ with the radius of $s_{i}$, i.e., $\mathcal B_{i} = \{A_{i}: \| A_{i} \| \le s_{i}\}$. Furthermore, suppose the input data matrix $X$ is restricted in a ball centred at $0$ with the radius of $B$, i.e., $\| X \| \le B$. Suppose $F$ is a hypothesis function computed by the neural network. If we define:
		\begin{equation}
		\mathcal H = \{ F(X): A_{i} \in \mathcal B_{i}, A^{u,v,s}_{t} \in \mathcal B^{u,v,s}_{t} \},
		\end{equation}
		where $i = 1, \ldots, L$, $(u,v,s)\in I_{V}$, and $t \in \{1, \ldots, L^{u,v,s}\}$. Let $\varepsilon = \sum_{j = 1}^{L}\varepsilon_{j}\rho_{j}\prod_{l = j+1}^{L}\rho_{l}s_{l}$. Then we have the following inequality:
		\begin{align}
		\label{formulaCoverChainBoundGeneral}
		\mathcal N(\mathcal H) \le  \prod_{i=1}^{L} \sup_{\mathbf A_{i-1} \in \bm{\mathcal B}_{i-1}} \mathcal N_{i}, 
		\end{align}
		where $\mathbf A_{i-1} = (A_{1}, \ldots, A_{i-1})$, $\bm{\mathcal B}_{i-1} = \mathcal B_{1} \times \ldots \times \mathcal B_{i-1}$, and
		\begin{equation}
		\mathcal N_{i}  = \mathcal N \left( \left\{ A_{i}F_{\mathbf A_{i-1}}(X): A_{i} \in \mathcal B_{i} \right\} \varepsilon_{i}, \| \cdot \| \right).
		\end{equation}
	\end{lemma}
	
	\subsection{Covering bound of FSTRN}
	
	This subsection gives a detailed proof for the covering bound of FSTRN. We first recall a result by Bartlett et al. \cite{bartlett2017spectrally}.
	\begin{lemma}[cf. \cite{bartlett2017spectrally}, Lemma 3.2]
		\label{matrixCover}
		Let conjugate exponents $(p, q)$ and $(r, s)$ be given with $p \le 2$, as well as positive reals $(a, b, \varepsilon)$ and positive integer $m$. Let matrix $X \in \mathbb R^{n \times d}$ be given with $\| X \|_{p} \le b$. Let $\mathcal H_{A}$ denote the family of matrices obtained by evaluating $X$ with all choices of matrix $A$:
		\begin{equation}
		\mathcal H_{A} \triangleq \left\{ XA | A \in \mathbb R^{d \times m}, \|A\|_{q, s} \le a \right\}.
		\end{equation}
		Then
		\begin{equation}
		\label{matrixCovBound}
		\log \mathcal N \left( \mathcal H_{A}, \varepsilon, \| \cdot \|_{2} \right) \le \ceil*{\frac{a^{2}b^{2}m^{2/r}}{\varepsilon^{2}}} \log(2dm).
		\end{equation}
	\end{lemma}
	This covering bound constrains the hypothesis complexity contributed by a single weight matrix.
	
	
	As Figure \ref{fig:SR_architecture} shows, suppose all hypothesis functions $F_{0}^{L}, F_{1}^{L}, \ldots, F_{D}^{L}, F_{Up}^{L}, F_{SR}^{L}$ respectively constitute a series of hypothesis spaces $\mathcal H_{0}^{L}, \mathcal H_{1}^{L}, \ldots, \mathcal H_{D}^{L}, \mathcal H_{Up}^{L}, \mathcal H_{SR}^{L}$. For the brevity, we rewrite those notations respectively as $F_{0}^{L}, F_{1}^{L}, \ldots, F_{D}^{L}, F_{D+1}^{L}, F_{D+2}^{L}$, and $\mathcal H_{0}^{L}, \mathcal H_{1}^{L}, \ldots, \mathcal H_{D}^{L}, \mathcal H_{D+1}^{L}, \mathcal H_{D+2}^{L}$.
	Also, suppose the covering number respectively with the radiuses $\varepsilon_{0}^{L}, \varepsilon_{1}^{L}, \ldots, \varepsilon_{D}^{L}, \varepsilon_{D+1}^{L}, \varepsilon_{D+2}^{L}$ are $\mathcal N(\mathcal H_{0}^{L}), \mathcal N(\mathcal H_{1}^{L}), \ldots, \mathcal N(\mathcal H_{D}^{L}), \mathcal N(\mathcal H_{D+1}^{L}), \mathcal N(\mathcal H_{D+2}^{L})$.
	
	\begin{proof}[Proof of Theorem \ref{thmCovBoundFSTRN}]
		Employing Lemma \ref{coverBoundGeneral}, we can straight obtain the following inequality.
		\begin{equation}
		\label{inequalityGeneral}
		\log \mathcal N(\mathcal H) \le \sum_{d=0}^{D}\log \mathcal N(\mathcal H_{d}^{L}).
		\end{equation}
		Applying eq. (\ref{matrixCovBound}) of Lemma \ref{matrixCover}, we can obtain the following result. We first calculate the covering bound of FRBs. Denote the PReLU in the $d$-th FRB as $\sigma^{d}$ and denote the weight matrices corresponding to the $2$ convolutional layers respectively as $A_{1}^{d}$ and $A_{2}^{d}$. Then, for $d = 1, \ldots, D$,
		\begin{align}
		\label{FRBGenCovBound}
		& \log \mathcal N(\mathcal H_{d + 1}) \nonumber\\
		\le & \frac{(b_{1}^{d+1})^{2} \| 		\sigma^{d}(F_{d}(X^{T})^{T}) \|_{2}^{2}}{(\varepsilon_{1}^{d+1})^{2}} \log(2W^{2}) \nonumber\\
		& + \frac{(b_{2}^{d+1})^{2} \| A_{1}^{d+1}\sigma^{d+1}(F_{d}(X^{T})^{T}) \|_{2}^{2}}{(\varepsilon_{2}^{d+1})^{2}} \log(2W^{2}).
		\end{align}
		Apparently,
		\begin{equation}
		\label{normEq1}
		\|\sigma^{d+1}(F_{d}(X^{T})^{T})\|_{2}^{2} \le (\rho^{d+1})^{2} \|F_{d}(X^{T})^{T}\|_{2}^{2},
		\end{equation}
		and
		\begin{align}
		\label{normEq2}
		& \|A_{1}^{d+1}\sigma^{d+1}(F_{d}(X^{T})^{T})\|_{2}^{2} \nonumber\\
		\le & (s_{1}^{d+1})^{2} \|\sigma^{d+1}(F_{d}(X^{T})^{T})\|_{2}^{2} \nonumber\\
		\le & (s_{1}^{d+1} \rho^{d+1})^{2} \|F_{d}(X^{T})^{T}\|_{2}^{2}.
		\end{align}
		Also, motivated by the proof of Lemma 4.3 of \cite{he2019why}, we can obtain the following equations.
		\begin{equation}
		\label{epEq1}
		\varepsilon_{1}^{d+1} = \varepsilon_{d}^{L} \rho^{d+1},
		\end{equation}
		\begin{equation}
		\label{epEq2}
		\varepsilon_{2}^{d+1} = \varepsilon_{1}^{d+1} (1 + s^{d+1}_{1}) = \varepsilon_{d}^{L} \rho^{d+1} (1 + s^{d+1}_{1}),
		\end{equation}
		and
		\begin{align}
		\label{epEq3}
		\varepsilon_{d+1}^{L} & = \varepsilon_{2}^{d+1} (1 + s^{d+1}_{2}) \nonumber\\
		& = \varepsilon_{d}^{L} \rho^{d+1} (1 + s^{d+1}_{1}) (1 + s^{d+1}_{2}).
		\end{align}
		Applying eqs. (\ref{normEq1}), (\ref{normEq2}), (\ref{epEq1}), (\ref{epEq2}), (\ref{epEq3}) to eq. (\ref{FRBGenCovBound}), we can obtain a covering bound for FRBs as follows.
		\begin{align}
		\label{FRBCovBound}
		& \log \mathcal N(\mathcal H_{d + 1}) \nonumber\\
		\le & \frac{\| F_{d}(X) \|^{2}_{2}}{(\varepsilon_{d+1}^{L})^{2}}\log \left(2W^{2}\right) (\rho^{d+1})^{2} \nonumber\\
		& \left[ (b_{1}^{d+1})^{2}(1+s_{1}^{d+1})^{2} + (b_{2}^{d+1})^{2}(s_{1}^{d+1})^{2} \right].
		\end{align}
		
		By applying eq. (\ref{normEq2}) and the induction method, we can straight get the following inequality:
		\begin{align}
		\label{inequalityNdp1}
		& \log \mathcal N(\mathcal H_{d + 1}) \nonumber\\
		\le & \prod_{i = 1}^{d} \left[ \left( \rho^{i} s_{1}^{i} s_{2}^{i} \right)^{2} + 1 \right] \left[ \left(b_{1}^{d}\right)^{2} \left(1 + s_{2}^{d}\right)^{2} + \left(b_{2}^{d}s_{1}^{d} \right)^{2} \right] \nonumber\\
		& \left(\frac{\| X \|_{2}s_{1}\rho^{d}}{\varepsilon^{d}}\right)^{2}.
		\end{align}
		Similarly, we can also get the following inequalities.
		\begin{equation}
		\label{inequalityN1}
		\log \mathcal N(\mathcal H_{1}) \le \frac{b_{1}^{2} \| X \|^{2}_{2}\bar\alpha}{\varepsilon^{2}} \log\left(2W^{2}\right),
		\end{equation}
		\begin{align}
		\label{inequalityNDp1}
		\log \mathcal N(\mathcal H_{D+1}) \le & \left\{ 1 + \sum_{i = 1}^{D}\prod_{j=1}^{i}\left[ (\rho^{j})^{2} (s^{j}_{1}s^{j}_{2})^{2} + 1 \right] \right\} \nonumber\\
		& \| X \|^{2}_{2} \rho_{1}^{2} \frac{b_{2}^{2}}{\varepsilon_{2}^{2}} \log\left(2W^{2}\right) \frac{b_{2}^{2}}{\varepsilon_{2}^{2}},
		\end{align}
		\begin{align}
		\label{inequalityNDp2}
		\log \mathcal N(\mathcal H_{D+2}) \le & \left\{ 1 + \sum_{i = 1}^{D}\prod_{j=1}^{i}\left[ (\rho^{j})^{2} (s^{j}_{1}s^{j}_{2})^{2} + 1 \right] \right\} \nonumber\\
		& \| X \|^{2}_{2} \rho_{1}^{2} \frac{b_{2}^{2}}{\varepsilon_{2}^{2}} \log\left(2W^{2}\right) s_{2}^{2} \frac{b_{3}^{2}}{\varepsilon_{3}^{2}} \nonumber\\
		& + \frac{b_{1}^{2} \| X \|^{2}_{2}}{\varepsilon^{2}}\log\left(2W^{2}\right).
		\end{align}
		Applying eqs. (\ref{inequalityNdp1}, \ref{inequalityN1}, \ref{inequalityNDp1}, and \ref{inequalityNDp2}) to eq. (\ref{inequalityGeneral}), we eventually prove the theorem.	
	\end{proof}
	
	\subsection{Generalization Bound for FSTRN}
	
	The Theorem 2 is the same as Theorem 4.4 of \cite{he2019why}. For the completeness of this paper, we restate the proof here.
	
	\begin{proof}[Proof of Theorem 2]
		We prove this theorem in $2$ steps: (1) First apply Lemma \ref{covNumBound} to get an upper bound on the Rademacher complexity; and then (2) Apply the result of (1) to Lemma \ref{GenBound} in order to get a generalization bound.
		
			(1) {\it Upper bound on the Rademacher complexity.}
			
			Applying eq. (\ref{formulaCovNumBound}) of Lemma \ref{covNumBound}, we can get the following inequality:
			\begin{align}
			\mathfrak R(\mathcal H_{\lambda}|_{D}) \le & \inf_{\alpha > 0} \left( \frac{4\alpha}{\sqrt{n}} + \frac{12}{n} \int_{\alpha}^{\sqrt n} \sqrt{\log \mathcal N(\mathcal H)} \text{d}\varepsilon \right) \nonumber\\
			\le & \inf_{\alpha > 0} \left( \frac{4\alpha}{\sqrt{n}} + \frac{12}{n} \int_{\alpha}^{\sqrt n} \frac{\sqrt{R}}{\varepsilon} \text{d}\varepsilon \right) \nonumber\\
			\le & \inf_{\alpha > 0} \left( \frac{4\alpha}{\sqrt{n}} + \frac{12}{n} \sqrt{R} \log\frac{\sqrt{n}}{\alpha} \right).
			\end{align}
			Apparently, the infinimum is reached uniquely at $\alpha = 3\sqrt{\frac{R}{n}}$. Here, we use a choice $\alpha = \frac{1}{n}$, and get the following inequality:
			\begin{equation}
			\mathfrak R(\mathcal H_{\lambda}|_{D}) \le \frac{4}{n^{\frac{3}{2}}} + \frac{18}{n} \sqrt{R} \log n.
			\end{equation}
			
			(2) {\it Upper bound on the generalization error.}
			
			Combining with Lemma \ref{GenBound}, we get the following inequality:
			\begin{align}
			& \Pr\{ \arg\max_{i} F(x)_{i} \ne y \} \nonumber\\
			\le & \hat{\mathcal R}_{\lambda}(F) + \frac{8}{n^{\frac{3}{2}}} + \frac{36}{n} \sqrt{R} \log n + 3 \sqrt{\frac{\log(1/\delta)}{2n}}.
			\end{align}
		The proof is completed.
	\end{proof}
	
	
	\section{Empirical Results}
	\label{sec:sup_exp}
	
	This appendix collects all empirical results omitted from the main text. Our algorithm outperforms the state-of-the-art methods in both qualitative and quantitative aspects.
	
	\subsection{Quantitatively Results}
	
	The quantitative results of all the methods on Vid4 \cite{DBLP:conf/cvpr/LiuS11} are summarized in Table \ref{tab:sup_psnr_ssim}, where the evaluation measures are the PSNR and SSIM indices. As demonstrated in Table \ref{tab:sup_psnr_ssim}, our algorithm has excellent robustness in different scenarios and outperforms all other methods.
	
	\begin{table*}[htbp]
		\centering
		\begin{tabular}{|c|c|c|c|c|c|c|}
			\hline
			\multirow{2}*{Methods}&City&Calendar&Walk&Foliage&Average\\
			& PSNR / SSIM & PSNR / SSIM & PSNR / SSIM & PSNR / SSIM & PSNR / SSIM\\
			\hline
			\hline
			Bicubic&24.82 / 0.58&19.98 / 0.55&25.33 / 0.78&22.91 / 0.54&23.25 / 0.62\\
			SRCNN\cite{DBLP:conf/eccv/DongLHT14}&25.46 / 0.65&21.08 / 0.65&27.16 / 0.84&24.05 / 0.66&24.47 / 0.71\\
			SRGAN\cite{DBLP:conf/cvpr/LedigTHCCAATTWS17}&25.30 / 0.64&21.04 / 0.64&26.55 / 0.81&23.69 / 0.62&24.16 / 0.68\\
			RDN\cite{DBLP:conf/icip/XuCSD18}&25.59 / 0.66&20.99 / 0.63&27.19 / 0.83&24.05 / 0.66&24.49 / 0.70\\
			BRCN\cite{DBLP:conf/nips/HuangWW15}&25.46 / 0.64&21.10 / 0.64&27.06 / 0.84&24.03 / 0.65&24.44 / 0.70\\
			VESPCN\cite{DBLP:conf/cvpr/CaballeroLAATWS17}&25.55 / 0.66&21.07 / 0.65&27.17 / 0.84&24.08 / 0.67&24.50 / 0.71\\
			\textbf{FSTRN}(ours)&\textbf{25.76} / \textbf{0.68}&\textbf{21.36} / \textbf{0.68}&\textbf{27.57} / \textbf{0.85}&\textbf{24.21} / \textbf{0.67}&\textbf{24.76} / \textbf{0.72}\\
			\hline
			
		\end{tabular}
		\caption{Comparison of the PSNR and SSIM results on vid4 \cite{DBLP:conf/cvpr/LiuS11} sequences by Bicubic, SRCNN\cite{DBLP:conf/eccv/DongLHT14}, SRGAN\cite{DBLP:conf/cvpr/LedigTHCCAATTWS17}, RDN\cite{DBLP:conf/icip/XuCSD18}, BRCN\cite{DBLP:conf/nips/HuangWW15}, VESPCN\cite{DBLP:conf/cvpr/CaballeroLAATWS17}, and our FSTRN with scale factor $4$.}
		\label{tab:sup_psnr_ssim}
	\end{table*}
	
	\subsection{Qualitatitve Results}
	
	We also qualitatively compare our algorithm with several existing algorithms, Bicubic, SRCNN\cite{DBLP:conf/eccv/DongLHT14}, SRGAN\cite{DBLP:conf/cvpr/LedigTHCCAATTWS17}, RDN\cite{DBLP:conf/icip/XuCSD18}, BRCN\cite{DBLP:conf/nips/HuangWW15}, VESPCN\cite{DBLP:conf/cvpr/CaballeroLAATWS17}, and our FSTRN. The comparison experiments are all with scale factor $4$. The qualitative results also illustrate the excellent performance of our algorithm.
	
	\begin{figure*}[htbp]
		\centering
		\begin{minipage}[t]{\linewidth}
			\centering
			\includegraphics[width=1\linewidth]{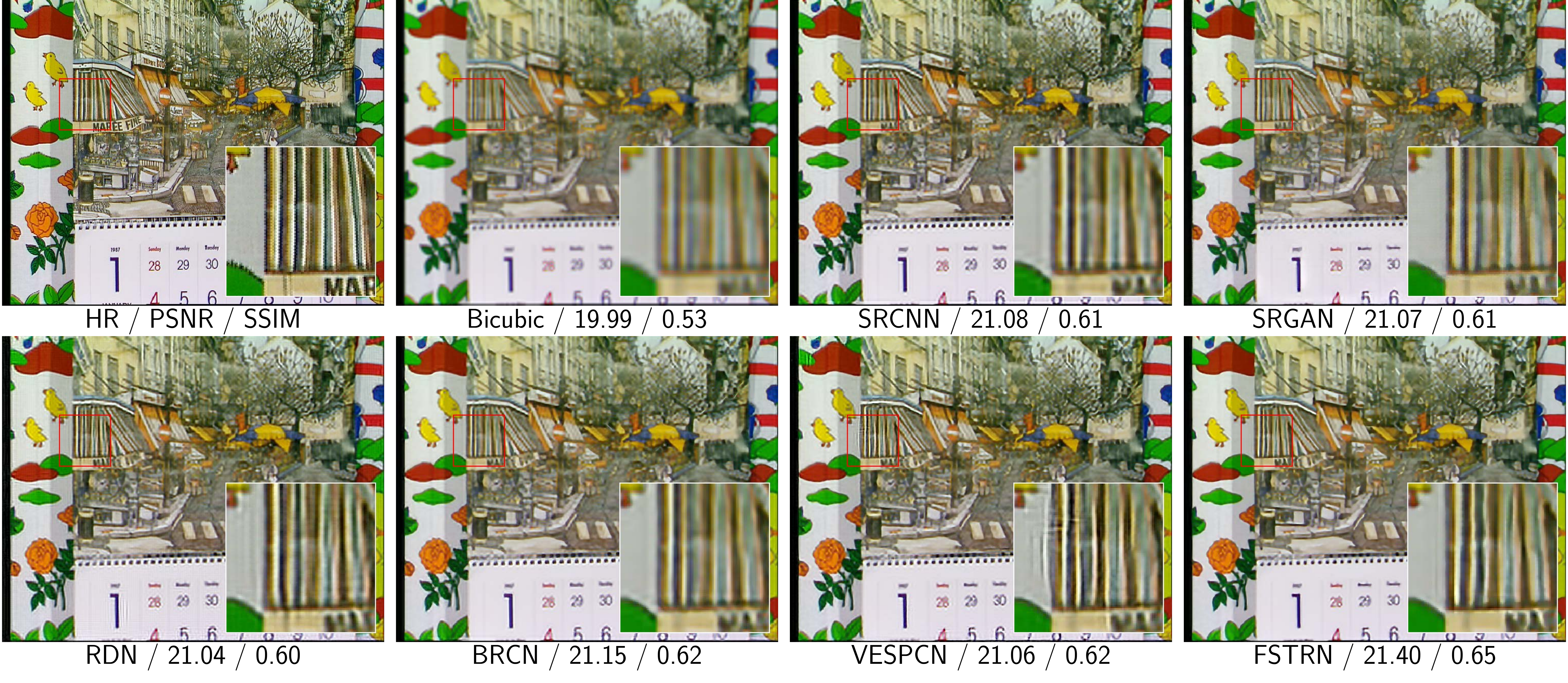} 
		\end{minipage}
		
		\caption{Visual comparisons of the super-resolution results for video \textbf{Calendar} on $\times 4$ upscaling factor.}
		\label{fig:calendar}
	\end{figure*}
	
	\vspace{3cm}
	
	\begin{figure*}[htbp]
		\centering
		\begin{minipage}[t]{\linewidth}
			\centering
			\includegraphics[width=1\linewidth]{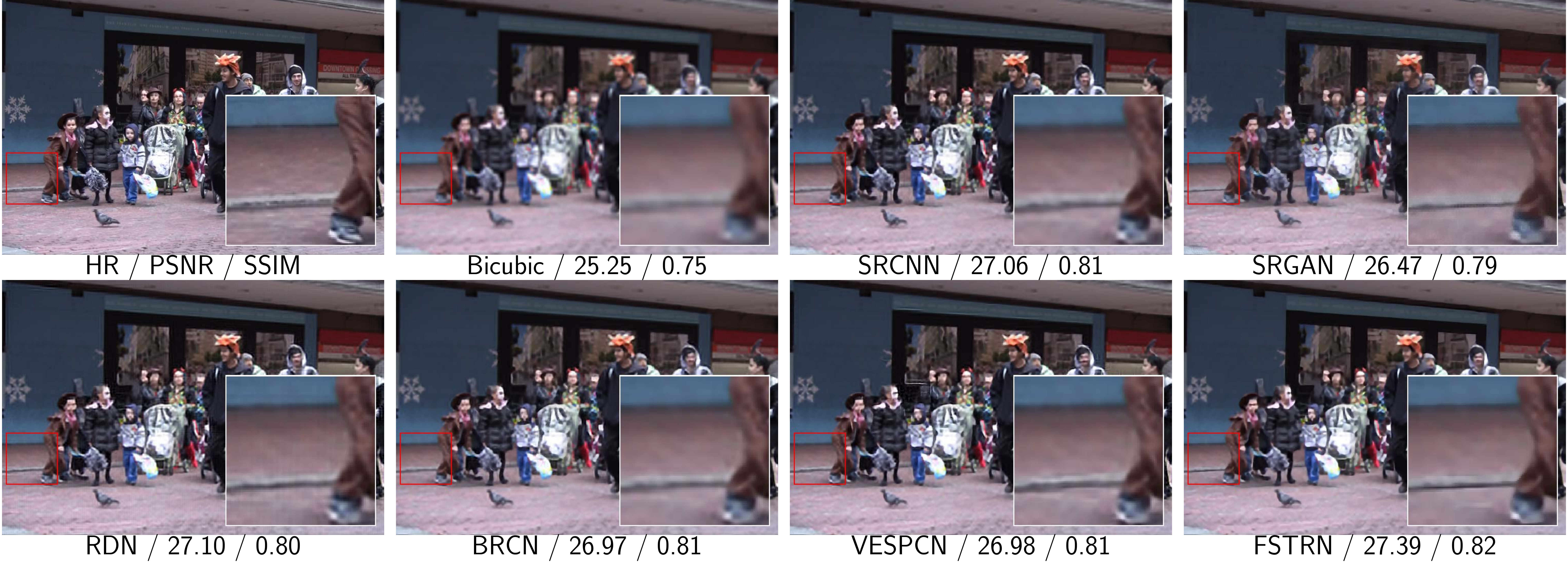}
		\end{minipage}
		
		\caption{Visual comparisons of the super-resolution results for video \textbf{Walk} on $\times 4$ upscaling factor.}
		\label{fig:walk}
	\end{figure*}
	
	\begin{figure*}[htbp]
		\centering
		\begin{minipage}[t]{\linewidth}
			\centering
			\includegraphics[width=1\linewidth]{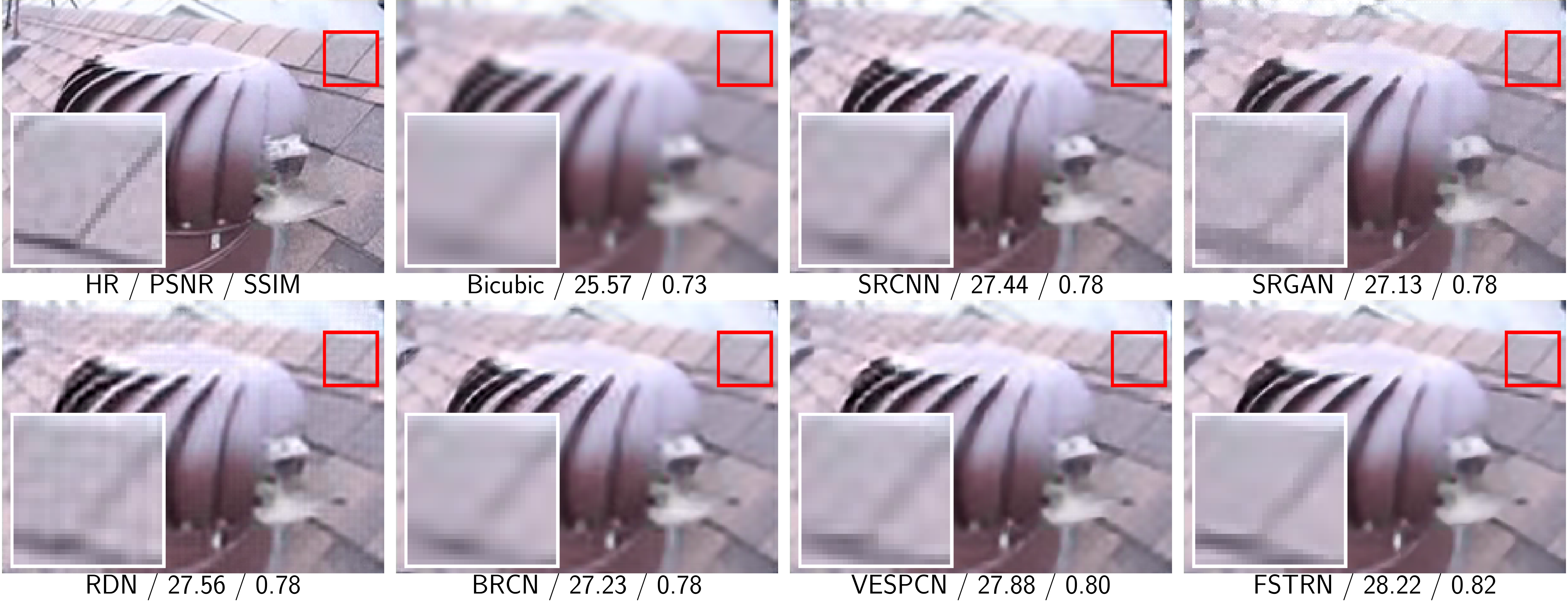}
		\end{minipage}
		
		\caption{Visual comparisons of the super-resolution results for video \textbf{Turbine} on $\times 4$ upscaling factor.}
		\label{fig:Turbine}
	\end{figure*}
	
	\begin{figure*}[htbp]
		\centering
		\begin{minipage}[t]{\linewidth}
			\centering
			\includegraphics[width=1\linewidth]{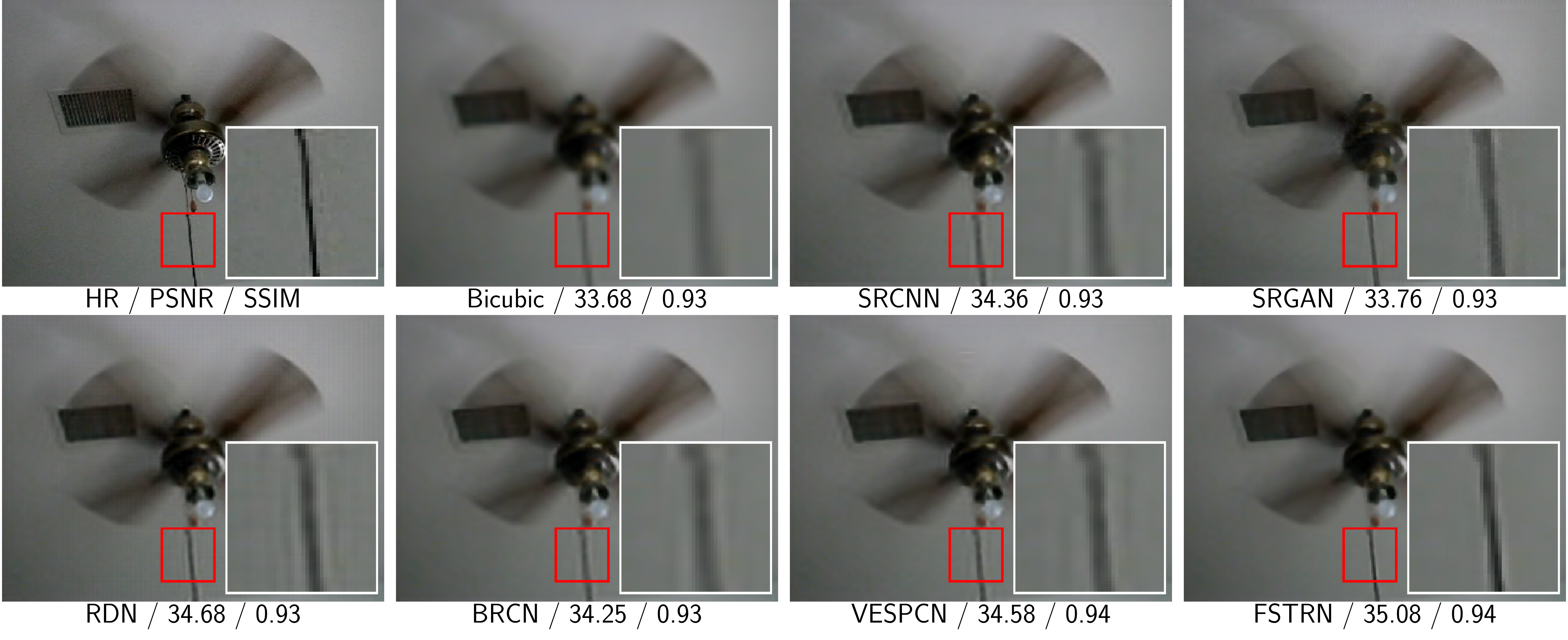}
		\end{minipage}
		
		\caption{Visual comparisons of the super-resolution results for video \textbf{Fan} on $\times 4$ upscaling factor.}
		\label{fig:Star}
	\end{figure*}

\end{appendix}

\end{document}